\documentclass[11pt]{article}
\usepackage{multirow} 
\def\MLE{\text{MLE}}
\def\TRACE{\text{Tr}}

\def\hbeta{\hat{\beta}}
\usepackage[numbers]{natbib}
\def\op{\text{op}}

\usepackage{amsmath, amssymb, amsthm}
\usepackage{mathrsfs}
\usepackage{geometry}
\usepackage{extarrows}
\usepackage{xcolor}
\usepackage{graphicx}
\usepackage[noend]{algpseudocode}
\usepackage[utf8]{inputenc}
\usepackage{natbib} 
\usepackage{booktabs}
\usepackage{enumitem}
\usepackage{titlesec}
\usepackage[colorlinks,
            linkcolor=red,
            anchorcolor=blue,
            citecolor=blue
            ]{hyperref}
\usepackage{authblk}
\usepackage{subcaption}
\usepackage{float}
\usepackage[title]{appendix}
\usepackage[linesnumbered,ruled,vlined]{algorithm2e}
\SetKwInput{KwInput}{Input}                
\SetKwInput{KwReturn}{Return}

\SetCommentSty{mycommfont}

\usepackage{array}
\newcolumntype{A}{>{\centering\arraybackslash}m{0.12\columnwidth}}
\newcolumntype{B}{>{\centering\arraybackslash}m{0.4\columnwidth}}

\DeclareMathOperator*{\argmax}{arg\,max}
\DeclareMathOperator*{\argmin}{arg\,min}
\def\bbeta{\boldsymbol{\beta}}

\title{Optimal Sample Selection Through Uncertainty Estimation and Its Application in Deep Learning}
\author{Yong Lin\thanks{ indicates equal contributions.} , Chen Liu$^*$, Chenlu Ye$^*$, Qing Lian, Yuan Yao, Tong Zhang}
\affil[]{
The Hong Kong University of Science and Technology}
\date{}

\usepackage{fullpage}

\def\bbE{\mathbb{E}}
\def\bbR{\mathbb{R}}
\def\bbP{\mathbb{P}}
\def\bbI{\mathbb{I}}

\def\bx{\textbf{x}}
\def\by{\textbf{y}}

\def\bI{\textbf{I}}
\def\bX{\textbf{X}}

\def\bV{\textbf{V}}
\def\bM{\textbf{M}}

\def\cD{\mathcal{D}}

\def\cL{\mathcal{L}}

\def\cN{\mathcal{N}}

\def\cS{\mathcal{S}}

\def\cX{\mathcal{X}}
\def\cY{\mathcal{Y}}

\def\bx{\boldsymbol{x}}
\def\by{\boldsymbol{y}}

\def\bmx{\mathbf{x}}
\def\bmy{\mathbf{y}}

\theoremstyle{plain}
\newtheorem{ass}{Assumption}
\theoremstyle{plain}
\newtheorem{thm}{Theorem}

\theoremstyle{plain}
\newtheorem{lemma}{Lemma}
\theoremstyle{plain}

\theoremstyle{plain}

\theoremstyle{plain}

\theoremstyle{plain}
\newtheorem{col}{Corollary}

\begin{document}

\maketitle

\noindent
\begin{abstract}
    Modern deep learning heavily relies on large labeled datasets, which often comse with high costs in terms of both manual labeling and computational resources. To mitigate these challenges, researchers have explored the use of informative subset selection techniques, including coreset selection and active learning.  Specifically, coreset selection involves sampling data with both input ($\bx$) and output ($\by$), active learning focuses solely on the input data ($\bx$).

    In this study, we present a theoretically optimal solution for addressing both coreset selection and active learning within the context of linear softmax regression. Our proposed method, COPS (unCertainty based OPtimal Sub-sampling), is designed to minimize the expected loss of a model trained on subsampled data. Unlike existing approaches that rely on explicit calculations of the inverse covariance matrix, which are not easily applicable to deep learning scenarios, COPS leverages the model's logits to estimate the sampling ratio. This sampling ratio is closely associated with model uncertainty and can be effectively applied to deep learning tasks. Furthermore, we address the challenge of model sensitivity to misspecification by incorporating a down-weighting approach for low-density samples, drawing inspiration from previous works.

To assess the effectiveness of our proposed method, we conducted extensive empirical experiments using deep neural networks on benchmark datasets. The results consistently showcase the superior performance of COPS compared to baseline methods, reaffirming its efficacy.
\end{abstract}

\section{Introduction}
In recent years, deep learning has achieved remarkable success in various domains, including computer vision (CV), natural language processing (NLP), reinforcement learning (RL) and autonomous driving, among others. However, the success of deep learning often relies on a large amount of labeled data. This requirement not only incurs expensive labeling processes  but also necessitates substantial computational costs. To address this challenge, an effective approach is to select an informative subset of the training data. Based on the selected subset, we can learn a deep neural network to achieve comparable performance with that trained on the full dataset.

There are two key types of problems related to this approach. The first is known as coreset selection \cite{har2005smaller, feldman2011unified, borsos2020coresets, zhou2022probabilistic}, which assumes that both the input data
$\{\bmx_i
\}_{i=1}^n$
and their corresponding labels $\{\bmy_i
\}_{i=1}^n$  are available for the full dataset 
$\cS = \{
\bmx_i,
\bmy_i\}_{i=1}^n$ containing 
$n$ samples. The objective here is to identify a subset of $\{\bmx_i,\bmy_i\}_{i=1}^n$ that significantly reduces the computation cost involved in training the models, thereby alleviating the computational burden. This problem is commonly referred to as the coreset selection problem. The second problem type, active learning,  assumes that only the input data $\{\bmx_i\}_{i=1}^n$ is accessible \cite{culotta2005reducing,ash2019deep,ren2021survey}, without the corresponding labels.  In this scenario, the aim is to selectively query the labels for a subset of $\{\bmx_i\}_{i=1}^n$. With the inquired labels, the neural network is trained on the selected subset. This problem is often referred to as active learning. Overall, these approaches provide promising solutions to mitigate the computational and labeling costs associated with training deep neural networks by intelligently selecting informative subsets of the data.

In this study, we theoretically derive the optimal sub-sampling scheme for both coreset selection and active learning in linear softmax regression. Our objective is to minimize the expected loss of the resulting linear classifier on the selected subset. The optimal sampling ratio is closely connected to the uncertainty of the data which has been extensively explored in reinforcement learning \cite{jin2020provably,dean2018regret,bubeck2012regret,lattimore2020bandit, ye2023corruption}. The detailed formulation and explanation of our sampling ratio is deferred to Section~\ref{eqn:optimal_derivation}.  We further show that the optimal sampling ratio is equivalent to the covariance of the output logits of independently trained models with proper scaling, which can be easily estimated in deep neural networks. We name our method as unCertainty based OPtimal Sub-sampling (COPS). 
While prior works such as \cite{ting2018optimal, wang2018optimal, imberg2020optimal} have explored related theoretical aspects, their approaches for estimating the sampling ratio are prohibitively expensive in the context of deep learning:  \cite{ting2018optimal} relies on the influence function of each data  which has been recognized as computationally demanding according to existing literature \cite{koh2017understanding};  \cite{wang2018optimal, imberg2020optimal} rely on the inverse of covariance matrix of input which is also computationally expensive due to the large dimensionality of the input data. There are also vast amount of literature on coreset selection and active learning, but few of them can claim optimality, which will be briefly reviewed in Section \ref{sect:related_works}.

We then conduct empirical experiments on real-world datasets  with modern neural architectures. Surprisingly, we find that directly applying COPS leads to bad performance which can be even  inferior to that of random sub-sampling. Upon conducting a thorough analysis of the samples selected by COPS, we observe a tendency for the method to excessively prioritize data exhibiting high uncertainty, i.e., samples from the low density region. Notably, existing literature has established that model estimation can be highly sensitive to misspecification issues encountered with low density samples \cite{he2022nearly, ye2023corruption}. It is important to note that the optimality of COPS is based on a well-specified linear model. Hence, this observation has motivated us to consider modifying COPS to effectively handle potential misspecification challenges. We use the short-hand notation $u_i$ to represent the uncertainty of $i$th sample, which is our original sampling ratio up to some scaling.
\cite{he2022nearly, ye2023corruption} show that applying the reweighting $\frac{1}{\max\{\alpha, u_i\}}$ to each sample during linear regression can make models more robust to misspecification, where $\alpha$ is a hyper-parameter. Thus, we simply borrow the idea and modify the sampling ratio $u_i$ by  $ \frac{u_i}{\max\{\alpha, u_i\}} \propto \min\{\alpha, u_i\}$.
We show the effectiveness of this modification by numerical simulations and real-world data experiments in Section~\ref{sect:thresholding_uncertainty}. 

In Section \ref{sect:main_experiments}, we conduct comprehensive experiments on several benchmark datasets, including SVHN, Places, and CIFAR10, using various backbone models such as ResNet20, ResNet56, MobileNetV2, and DenseNet121. Additionally, we verify the effectiveness of our approach on the CIFAR10-N dataset, which incorporates natural label noise. Furthermore, we extend our evaluation to include a NLP benchmark, IMDB, utilizing a GRU-based neural network. Across all these scenarios, our method consistently surpasses the baselines significantly, highlighting its superior performance. We summarize our contribution as follows
The contribution of this work can be summarized as follows:
\begin{itemize}
    \item \textbf{Theoretical derivation}: The study theoretically derives the optimal sub-sampling scheme for well-specified linear softmax regression. The objective is to minimize the expected loss of the linear classifier on the sub-sampled dataset. The optimal sampling ratio is found to be connected to the uncertainty of the data.
    \item \textbf{COPS method}: The proposed method, named unCertainty based OPtimal Sub-sampling (COPS), provides an efficient approach for coreset selection and active learning tasks. We show that the sampling ratio can be efficiently estimated using the covariance of the logits of independently trained models, which addresses the computational challenges faced by previous approaches \cite{ting2018optimal, wang2018optimal, imberg2020optimal}.
    \item \textbf{Modification to handle misspecification}: We empirically identified a potential issue with COPS, which overly emphasizes high uncertainty samples in the low-density region, leading to model sensitivity to misspecification. To address this, we draw inspiration from existing theoretical works \cite{he2022nearly,ye2023corruption} that downweight low-density samples to accommodate for the misspecification. By combining their techniques, we propose a modification to COPS that involves a simple thresholding of the sampling ratio. Both numerical simulations and real-world experiments demonstrate the significant performance improvements resulting from our straightforward modification.
    \item \textbf{Empirical Validation}: Empirical experiments are conducted on various CV and NLP benchmark datasets, including SVHN, Places, CIFAR10, CIFAR10-N and IMDB, utilizing different neural architectures including ResNet20, ResNet56, MobileNetV2,  DenseNet121 and GRU. The results demonstrate that COPS consistently outperforms baseline methods in terms of performance, showcasing its effectiveness.
\end{itemize}

\section{Related Works}
\label{sect:related_works}

\paragraph{Statistical Subsampling Methods.}
A vast amount of early methods adopts the statistical leverage scores to perform subsampling which is later used for ordinary linear regression \cite{drineas2012fast, drineas2011faster,ma2014statistical}. The leverage scores are estimated approximately \cite{drineas2012fast, clarkson2016fast} or combined with random projection \cite{meng2014lsrn}. These methods are relative computational expensive in the context of deep learning when the input dimension is large. Some recent works \cite{ting2018optimal,wang2018optimal,imberg2020optimal} achieves similar theoretical properties with ours. However,   \cite{ting2018optimal} is based on the influence function of each sample, which is computational expensive. \cite{ting2018optimal, imberg2020optimal} need to compute the inverse of covariance matrix, which is also impractical for deep learning. 

\paragraph{Active Learning.} This method designs acquisition functions to determine the importance of unlabeled samples and then trains models based on the selected samples with the inquired label \cite{ren2021survey}. There are mainly uncertainty-based and representative-based active learning methods. \textbf{Uncertainty-based} methods select samples with higher uncertainty, which can mostly reduces the uncertainty of the target model \cite{ash2019deep}. They design metrics such as entropy \cite{wang2014new,citovsky2023leveraging}, confidence \cite{culotta2005reducing},  margin \cite{joshi2009multi, roth2006margin}, predicted loss \cite{yoo2019learning} and gradient \cite{ash2019deep}.  Some recent works leverage variational autoencoders and adversarial networks to identify samples that are poorly represented by correctly labeled data \cite{sinha2019variational, kim2021task}. Some of these works provide theoretical guarantees expressed as probabilistic rates, but they do not claim to achieve optimality \cite{ting2018optimal}. The uncertainty-related technique has also been extended to RL \citep{he2022nearly,ye2023corruption}. \textbf{Representative-based} methods are also known as the diversity based methods \cite{citovsky2023leveraging}. They try to find samples with the feature that is most representative of the unlabeled dataset \cite{wei2015submodularity, sener2017active}. \cite{wei2015submodularity} casts the problem of finding representative samples as submodular optimization problem. \cite{sener2017active} tries to find the representative sampling by clustering, which is later adopted in \cite{ash2019deep}.   

\paragraph{Coreset Selection.} This method aims to find a subset that is highly representative of the entire labeled dataset. Some early works have focused on designing coreset selection methods for specific learning algorithms, such as SVM \cite{tsang2005core}, logistic regression \cite{huggins2016coresets}, and Gaussian mixture models \cite{lucic2017training}. However, these methods cannot be directly applied to deep neural networks (DNNs). To address this limitation, a solution has been proposed that leverages bi-level optimization to find a coreset specifically tailored for DNNs \cite{borsos2020coresets}. This approach has been further enhanced by incorporating probabilistic parameterization \cite{zhou2022probabilistic, zhou2022model}.
Another line of recent research efforts have aimed to identify coreset solutions with gradients that closely match those of the full dataset \cite{mirzasoleiman2020coresets, killamsetty2021grad, ash2019deep}.

\section{Theoretical Analysis}
\label{sect:theory}
\paragraph{Notation.} We use bold symbols $\bx$ and $\by$ to denote random variables and use $\bmx$ and $\bmy$ to denote deterministic values. Consider the $d$-dimensional vector $\bx \in \cX$ and the categorical label $\by \in \cY = \{c_0, c_1, \dots, c_K\}$.  Denote the joint distribution $(\bx, \by)$ as $\cD$. For any matrix $\bX \in \bbR^{d_1 \times d_2}$, define $\| \bX \|_{\op},~\| \bX \|_N$, and $\| \bX \|_F$ to be its $l_2$ operator norm, nuclear norm and Frobenius norm, respectively. 
The vectorized version of $\bX$ is denoted as $\mbox{\textbf{Vec}}(\bX)= (X_1^{\top}, X_2^{\top}, \ldots, X_{d_2}^{\top})^{\top}$, where $X_j$ is the $j$-th column of $\bX$.   Let $\cS$ denote the dataset containing $n$ labeled samples, i.e., $\cS := \{\bmx_i, \bmy_i\}_{i=1}^n$. We use $\cS_X$ to denote the unlabeled dataset $\cS := \{\bmx_i\}_{i=1}^n$. Let $\otimes$ denote the Kronecker product. 
For a sequence of random variables $X_1,X_2,\ldots$, we say that $X_n=o_P(1)$ if $X_n\rightarrow0$ as $n\rightarrow\infty$, and $X_n=O_P(1)$ if for all $\epsilon > 0$, there exists an $M$ such that $\sup_{n>N}\bbP(X_n>M)<\epsilon$.


\subsection{Optimal Sampling  in Linear Softmax Regression}
\label{eqn:optimal_derivation} 
Consider a $K$-class categorical response variable $\by\in\{c_0,c_1,\ldots,c_K\}$ and a $d$-dimensional covariate $\bx$. The conditional probability of $\by=c_k$ (for $k = 0, 1, \dots, K$) given $\bx$ is
\begin{align}
    \label{eqn:prob_k}
   p_k(\beta;\bx) = \frac{\exp({\bx^\top\beta_k})}{\sum_{l=0}^K\exp({\bx^\top\beta_l})}, 
\end{align}
where $\beta_k,~k = 0, 1, \dots, K$ are unknown regression coefficients belonging to a compact subset of $\bbR^d$. Following \cite{yao2019optimal}, we assume $\beta_0 = 0$ for identifiability.  
We further denote $\beta=(\beta_1^\top,\ldots,\beta_K^\top)^\top \in \mathbb R^{Kd}$. We use the bold symbol $\bbeta$ to denote the $d$-by-$K$ matrix $(\beta_1,\ldots,\beta_K)$.
In the sequel, we first derive the optimal sub-sampling schemes for both coreset selection and active learning  in linear softmax regression which minimize the expected test loss.   
Suppose the model is well-specified such that there exists an  true parameter $\beta^*\in\bbR^{Kd}$ with $\bbP(\by=c_k|\bx)~=~p_k(\beta^*; \bx)$ for all $\bx$ and $k$. 
Define $\delta_k(\by) := \bbI(y = c_k)$ where $\bbI$ is the indicator function.  
Let $\ell(\beta; \bx, \by)$ denote the cross entropy loss on the sample $(\bx, \by)$ as 
\begin{align}
    \ell(\beta; \bx, \by) = -\sum_{k=0}^K \delta_{k}(\by) \log p_k(\beta; \bx) = \sum_{k=1}^K\left[-\delta_{k}(\by)\bx^\top\beta_k + \log\{1+\sum_{l=1}^K \exp({\bx^\top\beta_l})\}\right].
\end{align}
We calculate the gradient and the hessian matrix of the loss function as follows: 
\begin{align}
    \label{eqn:first_derivative}
    \frac{\partial \ell(\beta; \bx, \by)}{\partial \beta} = -s(\beta; \bx, \by) \otimes \bx, \quad \mbox{ and } \quad \frac{\partial^2 \ell(\beta; \bx, \by)}{\partial \beta^2} = \phi(\beta; \bx) \otimes (\bx \bx^\top).
\end{align}
Here  $s(\beta; \bx, \by)$ is a $K$-dimensional vector with each element  $s_k(\beta; \bx, \by) = \delta_k(y) - p_k(\beta;\bx)$ for $k=1, ..., K$; and $\phi(\beta; \bx)$ is a $K \times K$ matrix with element  
  \begin{align}
    \label{eqn:phi_defi}
      \phi_{kk}(\beta; \bx) = p_k(\beta; \bx) - p_k(\beta; \bx)^2, \phi_{k_1k_2}(\beta; \bx) = -p_{k_1}(\beta; \bx) p_{k_2}(\beta; \bx),
  \end{align}
where $k, k_1, k_2 = 1, ..., K$ and $k_1 \neq k_2$. We further define the $K \times K$ matrix $\psi(\beta; \bx, y) := s(\beta; \bx, y)s(\beta; \bx, y)^\top$. For $k_1, k_2 = 1, ..., K$, we have 
\begin{align}
    \label{eqn:psi_defi}
    \psi_{k_1k_2}(\beta; \bx, \by) = [\delta_{k_1}(\by) - p_{k_1}(\beta; \bx)][\delta_{k_2}(\by) - p_{k_2}(\beta; \bx)].
\end{align}
We show $\bbE_{\by}[\psi(\beta^*; \bx, \by)|\bx] =  \phi(\beta^*; \bx)$ in Lemma \ref{lemma:phi_kesi_equal}. We use $\cL(\beta; \cD)$ to denote the expected cross-entropy loss on the distribution $\cD$ as
\begin{align}
\cL(\beta;\cD) = \bbE_{(\bx, \by) \sim \cD} [\ell(\beta; \bx, \by)],
\end{align}
It is easy to know that $\beta^* = \argmin_{\beta \in \bbR^{Kd}} \cL(\beta;\cD).$
Given the dataset $\cS = \{(\bmx_i, \bmy_i)\}_{i=1}^n$, we use $\cL(\beta;\cS)$ to denote the cross entropy loss of $\beta$ on $\cS$, i.e., 
\begin{align}
\label{eqn:dataset_loss}
    \cL(\beta;\cS) := \frac{1}{n}\sum_{(\bmx, \bmy) \in \cS} \ell(\beta; \bmx, \bmy).
\end{align}
We further use $\cL(\beta)$ to denote $\cL(\beta;\cS)$ when it is clear from the context.  Recall that $\cL(\beta;\cS)$ is the negative likelihood achieved by $\beta$ on $\cS$, then the maximum log-likelihood estimation (MLE) solution of $\beta$ on $\cS$ is 
\begin{align*}
    \hat \beta_\MLE := \argmin_{\beta \in \bbR^{Kd}} \cL(\beta;\cS).
\end{align*}
We further define 
\begin{align*}
      \bM_{X}(\beta; \cD) &  := \frac{\partial^2\cL(\beta; \cD)}{\partial^2\beta} =  \bbE_{(\bx, \by) \sim \cD} [\phi(\beta; \bx) \otimes (\bx \bx^\top)], \\
      \bM_{X}(\beta; \cS) &  := \frac{\partial^2\cL(\beta; \cS)}{\partial^2\beta} = \frac{1}{n} \sum_{(\bmx, \bmy) \in \cS} [\phi(\beta; \bmx) \otimes (\bmx \bmx^\top)].  
\end{align*}




\noindent \paragraph{Coreset Selection.} First, we focus on the coreset selection problem, assuming that we have access to the entire labeled dataset, i.e., $\cS=\{\bmx_i, \bmy_i\}_{i=1}^n$. We assign an sampling $\pi(\bmx, \bmy)$ to each samples in $\cS$ and then randomly select a subset of size $r$ according to $\pi(\bmx, \bmy)$. Denote the selected subset as $\bar \cS = \{\bar \bmx, \bar y\}$. We then estimate the parameter $\bar \beta$ based on the weighted loss 
\begin{align}
    \label{eqn:weighted_solver}
    \bar{\beta} =  \argmin_\beta \left(-\frac{1}{r} \sum_{(\bar\bmx, \bar \bmy) \in \bar \cS } \frac{1}{\pi(\bar \bmx, \bar \bmy)} \left( \sum_{k=1}^K\delta_{k}(\bar \bmy)\bar \bmx^\top\beta_k - \log\{1+\sum_{l=1}^K \exp({\bar \bmx^\top\beta_l})\} \right) \right),
\end{align}
We want the $\bar \beta$ estimated on the weighted sub-sampled dataset $\bar \cS$ to achieve low expected loss $\cL(\bar \beta; \cD)$. Omitting the higher order terms,  we are interested in the gap between the loss of $\bar \beta$ and $\beta^*$  as 
\begin{align}
    \cL(\bar \beta; \cD) - \cL(\beta^*; \cD) = \bbE_{(\bx, \by) \sim \cD} \left[ (\beta^* - \bar \beta)^\top \left( \phi(\beta; \bx) \otimes (\bx \bx^\top) \right) (\beta^* - \bar \beta) \right] 
\end{align}
Our goal is to find a sampling scheme parameterized by $\pi(\cdot)$ which minimizes the expectation of $\cL(\bar \beta; \cD) - \cL(\beta^*; \cD)$, i.e.,
\begin{align}
    \min_{\pi} \bbE_{\bar \cS|\cS, \pi}\left[\cL(\bar \beta; \cD) - \cL(\beta^*; \cD)\right], 
\end{align}
where the expectation is taking over the randomness in sampling based on $\pi(\cdot)$. 

\noindent \textbf{Active Learning}. For active learning problem, we have the unlabeled dataset $ \cS_X = \{\bmx_i\}_{i=1}^n$. We aim to assign a sampling weight $\pi(\bmx)$ for each sample $\bmx$ in $\cS_X$. Here we use the subscript $X$ in  $\pi_X$ to explicitly show that the sampling ratio in active learning only depends on $\bx$. When it is clear from the context, we also use $\pi(\bmx)$ to denote $\pi_X(\bmx)$ for simplicity.  Based on the sampled subset and queried label, which is also denoted as $\bar \cS = \{(\bar \bmx_i, \bar \bmy_i)\}_{i=1}^r$, we train the classifier $\bar \beta$ on the weighted loss as shown in Eqn~\eqref{eqn:weighted_solver} by replacing the weight $\pi(\bmx, \bmy)$ with $\pi(\bmx)$, i.e., 
\begin{align}
    \label{eqn:weighted_solver_active}
    \bar{\beta} = \argmin_\beta \left(-\frac{1}{r} \sum_{(\bar\bmx, \bar \bmy) \in \bar \cS } \frac{1}{\pi(\bar \bmx)} \left( \sum_{k=1}^K\delta_{k}(\bar \bmy)\bar \bmx^\top\beta_k - \log\{1+\sum_{l=1}^K \exp({\bar \bmx^\top\beta_l})\} \right) \right).
\end{align}
 Similar to that in coreset selection, we try to find a sampling scheme $\pi$ which optimizes the following equation: 
\begin{align}
    \min_{\pi} \bbE_{\bar \cS|\cS_X, \pi}\left[\cL(\bar \beta; \cD) - \cL(\beta^*; \cD)\right]. 
\end{align}

Before presenting the main theorem, we introduce two assumptions, which are standard in the subsampling literature \cite{wang2018optimal,yao2019optimal}.
\begin{ass}
    \label{ass:positive_M}
    The covariance matrix $\bM(\beta^*; \cS)$ goes to a positive definite matrix $\bM(\beta^*; \cD)$ in probability; and
    $n^{-2} \sum_{(\bmx, \bmy) \in \cS} \|\bmx \|^3 =O_p(1)$. 
\end{ass}
\begin{ass}
\label{ass:bounded_momentum}
    For $k=2,4$, $n^{-2} \sum_{(\bmx, \bmy) \in \cS} \pi(\bmx) \|\bmx\|^k = O_p(1)$; and there exists some $\delta > 0$ such that $n^{-(2+\delta)} \sum_{(\bmx, y) \in \cS} \pi(\bmx)^{-1-\delta} \|\bmx\|^{2+\delta} = O_p(1)$. 
\end{ass}
Assumption \ref{ass:positive_M} requires that the asymptotic matrix $\bM(\beta^*; \cD)$ is non-singular and $\bbE\|x\|^3$ is upper-bounded. Assumption \ref{ass:bounded_momentum} imposes conditions on both subsampling probability and covariates.
\begin{thm}[Optimal sampling in linear softmax regression]
\label{thm:multi_class_loss}
Suppose that the Assumptions \ref{ass:positive_M} and \ref{ass:bounded_momentum} hold.
\begin{itemize}
    \item[(a)] For coreset selection, the optimal sampling ratio of coreset selection that minimizes $\bbE_{\bar \cS|\cS, \pi}[\cL(\bar \beta; \cD) - \cL(\beta^*; \cD)]$ is 
\begin{align}
    \label{eqn:withy_optimal_sampling_multiclass}
    \pi(\bmx, \bmy) = \frac{\sqrt{\TRACE \left({\psi(\hat \beta_{\MLE}; \bmx, \bmy) \otimes (\bmx \bmx^\top)}  \bM_X^{-1}(\hbeta_\MLE; \cS) \right)}}{\sum_{(\bmx', \bmy') \in \cS} \sqrt{\TRACE \left({\psi(\hat \beta_{\MLE}; \bmx', \bmy') \otimes (\bmx' (\bmx')^\top)}  \bM_X^{-1}(\hbeta_\MLE; \cS) \right)} }.
\end{align}
\item[(b)] For active learning, the optimal sampling ratio of active learning that minimizes $\bbE_{\bar \cS|\cS_X, \pi}[\cL(\bar \beta; \cD) - \cL(\beta^*; \cD)]$ is 
\begin{align}
    \label{eqn:withouty_optimal_sampling_multiclass}
    \pi(\bmx) = \frac{\sqrt{\TRACE \left({\phi(\hat \beta_{\MLE}; \bmx) \otimes (\bmx \bmx^\top)}  \bM_X^{-1}(\hbeta_\MLE; \cS) \right)}}{\sum_{\bmx' \in \cS_X} \sqrt{\TRACE \left({\phi(\hat \beta_{\MLE}; \bmx') \otimes (\bmx' (\bmx')^\top)}  \bM_X^{-1}(\hbeta_\MLE; \cS) \right)} }.  
\end{align}
\end{itemize}
\end{thm}
The interpretation of the optimal sampling ratio will become clearer as we present the results for the binary logistic regression, which we will discuss later on. In the proof, by using the asymptotic variance of $\bar \beta$ first derived in \cite{wang2018optimal, yao2019optimal} and Taylor expansions, we can approximate the gap $\bbE_{\bar \cS|\cS, \pi}\left[\cL(\bar \beta; \cD) - \cL(\beta^*; \cD)\right]$ for coreset selection by
\begin{align}
        \bbE_{\bar \cS|\cS, \pi}\left[\cL(\bar \beta; \cD) - \cL(\beta^*; \cD)\right] = \frac{1}{rn^2}\sum_{(\bmx, \bmy) \in \cS} \frac{1}{\pi(\bmx, \bmy)}\TRACE \left({\psi(\hat \beta_{\MLE}; \bmx, \bmy) \otimes (\bmx \bmx^\top)}  \bM_X^{-1}(\hbeta_\MLE; \cS) \right),
\end{align}
and approximate the gap $\bbE_{\bar \cS|\cS_X, \pi}\left[\cL(\bar \beta; \cD) - \cL(\beta^*; \cD)\right]$ for active learning by
\begin{align}
    \bbE_{\bar \cS|\cS_X, \pi}\left[\cL(\bar \beta; \cD) - \cL(\beta^*; \cD)\right] = \frac{1}{rn^2}\sum_{\bmx \in \cS_X} \frac{1}{\pi(\bmx)}\TRACE \left({\phi(\hat \beta_{\MLE}; \bmx) \otimes (\bmx \bmx^\top)}  \bM_X^{-1}(\hbeta_\MLE; \cS) \right).
\end{align}
 Then, we obtain the minimizers of the two terms above with the Cauthy-Schwarz inequality separately. The detailed proof is in Appendix \ref{sect:proof_of_optimal_sampling}.
Note that distinct from \cite{wang2018optimal, yao2019optimal} that aim to reduce the variance of $\bar \beta$, we target on the expected generalization loss. However, directly computing our sampling ratio as well as those in \cite{wang2018optimal, yao2019optimal} is computationally prohibitive in deep learning, since they rely on the inverse of the covariance matrix. Whereas, as we will show later, our sampling ratio is closely connected to sample uncertainty and can be effectively estimated by the output of DNN. 

Now we illustrate the main intuition for the optimal sampling by considering the binary logistic classification problem as an example. In this case, we known that $K=1$, $\by \in \{c_0, c_1\}$, and $\beta=\beta_1 \in \bbR^{d}$. Correspondingly, the binary logistic regression model is in the following form:
\begin{align*}
    p_1(\beta; \bx) = \frac{\exp(\bx^\top \beta_1)}{1 + \exp(\bx^\top \beta_1)}.
\end{align*}
The covariance matrix becomes
$$
M_X = 1/n \sum_{(\bmx, \bmy) \in \cS}(p_1(\hbeta_\MLE; \bmx) - p_1(\hbeta_\MLE; \bmx)^2)\bmx \bmx^\top.
$$

\begin{col}[Logistic regression optimal sampling]
    Suppose that the Assumptions \ref{ass:positive_M} and \ref{ass:bounded_momentum} hold. 

    \begin{itemize}
        \item[(a)] For coreset selection, the optimal sampling ratio that minimizes $\bbE_{\bar \cS|\cS, \pi}[\cL(\bar \beta; \cD) - \cL(\beta^*; \cD)]$ is 
    \begin{align}
        \label{eqn:optimal_ratio_binary_coreset}
        \pi(\bmx, \bmy) = \frac{\left|\delta_1(y) - p_1(\hat \beta_\MLE; \bmx)\right| \|\bmx\|_{M_X^{-1}}}{\sum_{(\bmx', \bmy') \in \cS} \left|\delta_1(y) - p_1(\hat \beta_\MLE; \bmx')\right|\|\bmx'\|_{M_X^{-1}}}.
    \end{align}

    \item[(b)] For active  learning, the optimal sampling ratio that minimizes $\bbE_{\bar \cS|\cS_X, \pi}[\cL(\bar \beta; \cD) - \cL(\beta^*; \cD)]$ is 
    \begin{align}
        \pi(\bmx) = \frac{\sqrt{p_1(\hat \beta_\MLE; \bmx) - p_1(\hat \beta_\MLE; \bmx)^2} \|\bmx\|_{M_X^{-1}}}{\sum_{\bmx' \in \cS} \sqrt{p_1(\hat \beta_\MLE; \bmx) - p_1(\hat \beta_\MLE; \bmx)^2}\|\bmx'\|_{M_X^{-1}}}.
    \end{align}
    \end{itemize}
\end{col}
\paragraph{Intuition of the optimal sampling ratio.} The optimal ratio for coreset selection is proportional to 
$\left|\delta_1(y) - p_1(\hat \beta_\MLE; \bmx)\right|\cdot \|\bmx\|_{M_X^{-1}}$, which is decomposed of two components:
\begin{itemize}
    \item $\left|\delta_1(y) - p_1(\hat \beta_\MLE; \bmx)\right|$ is related to the prediction error of $\hat\beta_\MLE$. 
    \item $\|\bmx\|_{M_X^{-1}}$ has been widely explored in RL literature which is connected to uncertainty. Specifically, $\|\bmx\|^2_{M_X^{-1}}$ represents the inverse of the effective sample number in the $\cS$ along the $\bmx$ direction \cite{jin2020provably}. A larger $\|\bmx\|^2_{M_X^{-1}}$ indicates that there are less effective samples in the $\bmx$ direction. In this case, the prediction on $\bmx$ will be more uncertain. Therefore, $\|\bmx\|_{M_X^{-1}}$ is used to characterize the uncertainty along the $\bmx$ direction by . 
\end{itemize}
Samples with significant uncertainty and substantial prediction errors will result in a higher sampling weight for coreset selection. As for the active learning, $\left|\delta_1(y) - p_1(\hat \beta_\MLE; \bmx)\right|$ is replaced by 
$\sqrt{p_1(\hat \beta_\MLE; \bmx) - p_1(\hat \beta_\MLE; \bmx)^2}$ as we take conditional expectation over $\by$ since 
$$
{p_1(\hat \beta_\MLE; \bmx) - p_1(\hat \beta_\MLE; \bmx)^2} \approx \bbE_{\by|\bx=\bmx} (\delta_1(y) - p_1(\hat \beta_\MLE; \bmx))^2,
$$
as $n \rightarrow \infty$. $\sqrt{p_1(\hat \beta_\MLE; \bmx) - p_1(\hat \beta_\MLE; \bmx)^2}$ assigns large weights to those samples near the decision boundary. In summary, the optimal sampling ratios can be determined by weighting the uncertainty of samples with their corresponding prediction errors.

\subsection{Efficient approximation of the optimal sampling ratio}
There are some issues in estimating the optimal sampling ratio in Eqn~\eqref{eqn:withy_optimal_sampling_multiclass} and \eqref{eqn:withouty_optimal_sampling_multiclass}:
\begin{itemize}
    \item[(a)] We can not obtain $\hat \beta_\MLE$ in practice since it is solved on the whole dataset;
    \item[(b)]  Calculating the inverse of the covariance matrix $\bM_X(\hbeta_\MLE; \cS)$ is computationally prohibitive due to the high dimentionality in  deep learning. 
\end{itemize}
To solve the issue (a), \cite{wang2018optimal} proposes to fit a $\beta$ on the held out probe dataset $\cS'$ (a small dataset independent of $\cS$) to replace $\hat \beta_\MLE$. Whereas, the issue (b) remains to be the major obstacle for our method as well as those in \cite{ting2018optimal, yao2019optimal}. 
In the following part, we will by-pass the issue (b) by showing that ${\psi(\hat \beta_{\MLE}; \bmx, y) \otimes (\bmx \bmx^\top)}  \bM_X^{-1}(\hbeta_\MLE; \cS)$ is related to the standard deviation of the output logits from  independently trained models . 
To be more specific, we fit $M$ independent MLE linear classifiers $\{\hat\bbeta^{(m)}\}_{m=1}^M$ on $M$  probe datasets $\{\cS^{(m)}\}_{m=1}^M$ which is independent of $\cS$. We then show that for each sample $(\bmx, \bmy)$ in $\cS$, we can estimate ${\psi(\hat \beta_{\MLE}; \bmx, \bmy) \otimes (\bmx \bmx^\top)}  \bM_X^{-1}(\hbeta_\MLE; \cS)$ by the covariance of each model's logits i.e., $(\hat \bbeta^{(m)})^\top \bmx $, as shown in Eqn~\eqref{eqn:covariance_of_logits} of Algorithm~\ref{algo:uncertainty_estimation}. 

\begin{algorithm}[t]
\caption{Uncertainty estimation in linear softmax regression.  \label{algo:uncertainty_estimation}}
\KwInput{Probe datasets $\{\cS^{(m)}\}_{m=1}^M$, the sampling dataset $\cS$ for coreset selection or $\cS_X$ for active learning.}
\KwOut{The estimated uncertainty for each sample in $\cS$ or $\cS_X$.}
For $m = 1, ..., M$, solve $\hat \beta^{(m)} = \argmin_{\beta \in \bbR^{Kd}} \cL(\beta; \cS^{(m)})$. Denote  
\begin{align} \label{eqn:probe_betas}
    \Tilde{\beta} = \frac{1}{M} \sum_{m=1}^M \hat \beta^{(m)}, \mbox{ }\hat \bbeta^{(m)} = [\hat \beta^{(m)}_1, \hat \beta^{(m)}_2, ..., \hat \beta^{(m)}_K],  \mbox{ and } \Tilde{\bbeta} = \frac{1}{M} \sum_{m=1}^M \hat \bbeta^{(m)}.
\end{align} 

For each $\bmx$, obtain $ \{(\hat\bbeta^{(m)})^\top \bmx\}_{m=1}^M$ and the covariance of them:
\begin{align}
    \label{eqn:covariance_of_logits}
    \Sigma_M(\bmx) = \frac{1}{M - 1}\sum_{m=1}^M\left( \left(\hat \bbeta^{(m)}\right)^\top \bmx  - 
        \Tilde \bbeta ^\top \bmx \right) \left(\left(\hat \bbeta^{(m)}\right)^\top \bmx  - 
        \Tilde \bbeta ^\top \bmx \right)^\top.
\end{align}

Get the predicted probability of $\bmx$, i.e., $p(\Tilde{\beta}; \bmx)$,  as in Eqn~\eqref{eqn:prob_k}. Estimate the uncertainty for each sample as following:
\begin{itemize}
    \item Case (1) coreset selection. Obtain $\psi(\Tilde{\beta}; \bmx, \bmy)$ according to Eqn~\eqref{eqn:psi_defi} and obtain the uncertainty estimation as
    \begin{align*}
        u(\bmx, \bmy) = \TRACE\left(\psi(\Tilde\beta;\bmx, \bmy)\Sigma_M(\bmx)\right);
    \end{align*}
    \item Case (2) active learning. Obtain $\phi(\Tilde{\beta}; \bmx)$ according to Eqn~\eqref{eqn:phi_defi} and obtain the uncertainty estimation as 
    \begin{align*}
 u(\bmx) = \TRACE\left(\phi(\Tilde\beta;\bmx)\Sigma_M(\bmx)\right).
    \end{align*} 
\end{itemize}

\end{algorithm}

\begin{thm}[Uncertainty estimation in linear models]
\label{thm:uncertainty_estimation_multi}
Supposing that Assumptions \ref{ass:positive_M} and \ref{ass:bounded_momentum} hold, we have $M$ probe datasets $\{\cS^{(m)}\}_{m=1}^M$ and each $\cS^{(m)}$ contains $n'$ samples, we independently fit $M$ MLE classifiers $\{\hat \bbeta^{(m)}\}_{m=1}^M$ on $\{\cS^{(m)}\}_{m=1}^M$. Denote $\Tilde \beta = \frac{1}{M} \sum_{m=1}^M \mbox{Vec}(\hat \bbeta^{(m)}) $ and define $\Sigma_M(\bmx)$ as Eqn~\eqref{eqn:covariance_of_logits} in Algorithm~\ref{algo:uncertainty_estimation} 
, then  as $M \xrightarrow[]{} \infty$, $n' \xrightarrow[]{} \infty$ and $n\rightarrow\infty$, for  $(\bmx, \bmy) \in \cS$, we have 
\begin{align*}
    n' \TRACE\left(\psi(\Tilde \beta;\bmx, y)\Sigma_M(\bmx)\right) - \TRACE\left({\psi(\hat \beta_\MLE; \bmx, \bmy) \otimes (\bmx \bmx^\top)}  \bM_X^{-1}(\hat \beta_\MLE; \cS)\right) = o_P(1), \\
    n' \TRACE\left(\phi(\Tilde \beta;\bmx)\Sigma_M(\bmx)\right) -\TRACE\left({\phi(\hat \beta_{\MLE}; \bmx) \otimes (\bmx \bmx^\top)}  \bM_X^{-1}(\hat \beta_\MLE; \cS)\right) = o_P(1).
\end{align*}
\end{thm}
See Appendix \ref{app:proof_of_output_variance} for a proof. This theorem demonstrates that the uncertainty quantities can be approximated without explicitly calculating the inverse of covariance matrix. 
Instead, we only need to calculate a MLE estimator $\tilde \beta$ and the covariance of the output logits $\{(\hat \bbeta^{(m)})^\top \bmx\}_{m=1}^M$ derived from $M$ models. In other words, we only need to obtain $\{\hat \bbeta^{(m)}\}$ on $M$ probe sets, respectively. We then obtain the optimal sampling ratio through calculating $\Sigma_M(\bmx)$, which is the covariance of $\{(\hat \bbeta^{(m)})^\top \bmx \}_{m=1}^M$ as defined in Eqn~\eqref{eqn:covariance_of_logits}. 

\begin{algorithm}[t]
\caption{COPS  for coreset selection on linear models \label{algo:linear_those_coreset}}
\KwInput{Training data $\cS$, $M$ probe datasets $\{\cS^{(m)}\}_{m=1}^M$, sub-sampling size $r$. }
\KwOut{The selected subset $\bar\cS$ and the model $\bar \beta$.}

For each $(\bmx, \bmy) \in \cS$, obtain $u(\bmx, \bmy)$ by Algorithm~\ref{algo:uncertainty_estimation} with $\{\cS^{(m)}\}_{m=1}^M$;
 
Randomly draw $\bar\cS$ containing $r$ samples from $\cS$ by $\pi(\bmx, \bmy) = u(\bmx, \bmy) / \sum_{(\bmx', \bmy') \in \cS} u(\bmx', \bmy')$.

Solve $\bar \beta$ on the weighted subset $\bar\cS(\pi)$ according to Eqn~\eqref{eqn:weighted_solver}.
\end{algorithm}

\begin{algorithm}[t]
\caption{COPS  for active learning on linear models \label{algo:linear_those_active}}
\KwInput{Training data $\cS_X$, $M$ probe datasets $\{\cS^{(m)}\}_{m=1}^M$, sub-sampling size $r$. }
\KwOut{The selected subset $\bar\cS$ with inquired label and the model $\bar \beta$.}

For each $\bmx \in \cS_X$, obtain $u(\bmx)$ by Algorithm~\ref{algo:uncertainty_estimation} with $\{\cS^{(m)}\}_{m=1}^M$;
 
Randomly draw $\bar\cS_X$ containing $r$ samples from $\cS$ by $\pi(\bmx) = u(\bmx) / \sum_{\bmx' \in \cS} u(\bmx')$.

Obtain the labeled data set $\bar\cS$ by labeling each sample in $\bar\cS_X$.

Solve $\bar \beta$ on the weighted subset $\bar\cS(\pi)$ according to Eqn~\eqref{eqn:weighted_solver}.
\end{algorithm}

\paragraph{Approximations in Deep Learning.} Our objective is to develop a sub-sampling method for deep learning. Let's consider a deep neural network $f_\theta(\mathbf{x})$ with parameters $\theta \in \mathbb{R}^{d'}$, where both $d$ and $d'$ are extremely large in the context of deep learning. There exist gaps between the theory presented in Section~\ref{eqn:optimal_derivation} and deep learning due to the nonlinearity involved in $f_\theta$. However, we can leverage insights from learning theory, such as the Neural Tangent Kernel \cite{Jacot2018NeuralTK}, which demonstrates that a wide DNN can be approximated by a linear kernel with a fixed feature map $\nabla_\theta f_\theta(\cdot) \in \mathbb{R}^d \xrightarrow[]{\ } \mathbb{R}^{d'}$. Consequently, we can approximate uncertainty by calculating the standard deviation from different linear kernels, as outlined in Theorem \ref{thm:uncertainty_estimation_multi}. Importantly, our method does not necessitate explicit computation of the linear kernel, as we only require the output $\beta^\top \mathbf{x}$ from Theorem~\ref{thm:uncertainty_estimation_multi}. Thus, we can directly replace $\beta^\top \mathbf{x}$ with the output of the DNN, i.e., $f_\theta(\mathbf{x})$.

Let $f_{\theta, k}(\bx)$ denote the $k$th dimension of $f_\theta(\bx)$ for $k = 0, ..., K$. We denote the output probability of $f_\theta$ on sample $\bx$ by
\begin{align*}
    p(f_\theta; \bx) = [p_0(f_\theta; \bx), p_1(f_\theta; \bx), ..., p_K(f_\theta; \bx)], \mbox{ where } p_k(f_\theta; \bx) = \frac{\exp(f_{\theta, k}(\bx))}{\sum_{l=0}^K \exp(f_{\theta, l}(\bx))}.
\end{align*}
Recall in Algorithm~\ref{algo:uncertainty_estimation} that we train $M$ independent linear models on $M$ different probe sets, respectively. In practice, getting $M$ additional probe sets can be costly. One option is to use bootstrap, where $M$ subsets are resampled from a single probe set $\cS'$ and the variance is estimated based on the $M$ trained models. \cite{gonccalves2005bootstrap} shows that the variance estimated by bootstrap converges to the asymptotic variance, which is the uncertainty quantity. 
However, we adopt a different way which is more popular in deep learning: we train $M$ neural networks, $\{f_{\theta^{(m)}}\}_{m=1}^M$, on a single probe set $\cS'$ with different initialization and random seeds, which empirically outperforms the bootstrap method.  
With $\{f_{\theta^{(m)}}\}_{m=1}^M$, we then replace the linear models in Algorithm \ref{algo:uncertainty_estimation} by their DNN counterparts, i.e., replace $\hat\bbeta_{m}^\top  \bmx$ by $f_{\theta^{(m)}}(\bmx)$,  $\Tilde \bbeta^\top \bmx$ by $\frac{1}{M} \sum_{m=1}^M f_{\theta^{(m)}}(\bmx)$, and $p(\Tilde \beta; \bmx)$ by $\frac{1}{M} \sum_{m=1}^M p(f_{\theta^{(m)}}; \bmx)$.
We summarize the uncertainty estimation for DNN in Algorithm~\ref{algo:uncertainty_estimation_DNN} in Appendix~\ref{sect:uncertainty_estimation_dnns}. Notably, our method can be further simplified by training a single model on $\cS'$ with dropout and then can obtain $\{f_{\theta^{(m)}}\}_{m=1}^M$ by using Monte Carlo Dropout during inference. In Section~\ref{sect:main_experiments}, we also empirically compare different uncertainty estimation methods including different initialization, bootstrap, and dropout. 

The detailed algorithm as summarized in Algorithm~\ref{algo:dnn_those_coreset} and~\ref{algo:dnn_those_active} in Appendix~\ref{sect:uncertainty_estimation_dnns}.

\section{Towards Effective Sampling Strategy  in Real Word Applications}
\label{sect:thresholding_uncertainty}
In this section, we enhance the theoretically motivated sampling algorithm by incorporating insights gained from empirical observations. To begin, we experiment with the optimal sampling strategy Algorithm~\ref{algo:dnn_those_coreset} and~\ref{algo:dnn_those_active} on deep learning datasets.

\subsection{Vanilla uncertainty sampling strategy is ineffective in applications}

\paragraph{Settings.} We try out the sampling for DNN, i.e., Algorithm~\ref{algo:dnn_those_coreset} and \ref{algo:dnn_those_active} (with uncertainty estimation in Algorithm~\ref{algo:uncertainty_estimation_DNN}) with ResNet20~\cite{he2016deep}. 
We performed experiments on three datasets: (1) CIFAR10~\cite{krizhevsky2009learning}, (2) CIFARBinary, and (3) CIFAR10-N~\cite{wei2021learning}. CIFARBinary is a binary classification dataset created by selecting two classes (plane and car) from CIFAR10. CIFAR10-N is a variant of CIFAR10 with natural label noise~\cite{wei2021learning}. For a more comprehensive description of the datasets, please refer to Section \ref{sect:main_experiments}. For all settings, 
we split the training set into two subsets, i.e., the probe set ($\cS'$ in Algorithm~\ref{algo:dnn_those_coreset}-\ref{algo:dnn_those_active}) and the sampling dataset set ($\cS$ in Algorithm~\ref{algo:dnn_those_coreset}-\ref{algo:dnn_those_active}). We train 10 probe neural networks on $\cS'$ and estimate the uncertainty of each sample in $\cS$ with these networks. We select an subset with 300 samples per class from $\cS$ according to Algorithm~\ref{algo:dnn_those_coreset}-\ref{algo:dnn_those_active}, on which we train the a ResNet20 from scratch.

Since we conduct experiments on multiple datasets with different sub-sampling size, and for both coreset selection and active learning problems. We then use WithY to denote the coreset selection since we have the whole labeled dataset and we use WithoutY for active learning.  We use the triple ``(dataset name)-(target sub-sampling size)-(whether with $Y$)'' to denote an experimental setting, for example: CIFAR10-3000-WithY is short for the setting to select 3,000 samples from labeled CIFAR10 dataset for coreset selection.

\paragraph{Results.} Surprisingly, the results in Figure~\ref{fig:cifar_vanilla} shows that the sampling Algorithm~\ref{algo:dnn_those_coreset} and \ref{algo:dnn_those_active} are even inferior than uniform sampling in some settings both for coreset selection (WithY) and active learning (WithoutY).  For example, in the CIFARBinary-600-WithY setting in Figure~\ref{fig:cifar_vanilla}, uncertainty sampling leads to a testing performance of 75.26\%, which is much worse than uniform sampling's performance 88.31\%.  



\begin{figure}[t]
\centering
\includegraphics[width=10cm]{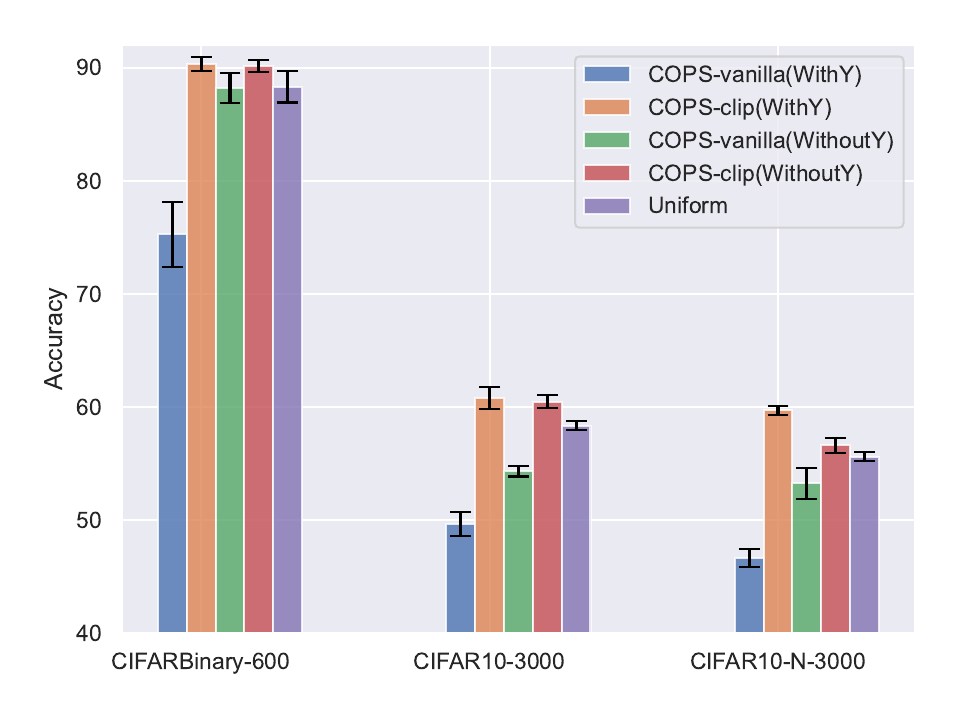}
\caption{The vanilla implementation of the uncertainty Algorithm~\ref{algo:dnn_those_coreset} and \ref{algo:dnn_those_active} (i.e., COPS-vanilla) displays inferior performance. Whereas, thresholding the maximum uncertainty during sample selection (i.e., COPS-clip) significantly enhances the overall performance. \label{fig:cifar_vanilla}}
\end{figure}






\paragraph{A closer look at the Uncertainty sampling.}  Figure~\ref{fig:cifar10_withy_uncertainty_histogram}(a) visualizes the uncertainty distribution of samples in CIFAR10 estimated by Algorithm~\ref{algo:uncertainty_estimation_DNN} . Figure~\ref{fig:cifar10_withy_uncertainty_histogram}(b) shows the uncertainty of the 3000 samples selected according to the sample selection ratio in Eqn~\eqref{eqn:withy_optimal_sampling_multiclass}, i.e., the uncertainty of 3000 samples selected by COPS in the CIFAR10-3000-WithY setting. The uncertainty distribution of the selected data in Figure~\ref{fig:cifar10_withy_uncertainty_histogram}(b) is quite different from the uncertainty distribution of the full dataset in Fig~\ref{fig:cifar10_withy_uncertainty_histogram}(a). The selected subset contains a large number of data with high uncertainty. Figure~\ref{fig:cifar10_withoy_uncertainty_histogram} shows similar trends in CIFAR10-3000-WithoutY.
 
Recall that the optimal sampling ratio is derived in a simplified setting where we assume that there is no model misspecification. The sampling schemes in Eqn.~\eqref{eqn:withy_optimal_sampling_multiclass}~and~\eqref{eqn:withouty_optimal_sampling_multiclass} tend to select samples from the low density region with high uncertainty. Whereas, previous studies \cite{he2022nearly, ye2023corruption} demonstrate that in cases where substantial misspecification happens to samples on low-density regions, the model estimation can be significantly impacted. We conjecture that the  uncertainty sampling methods in Algorithms~\ref{algo:dnn_those_coreset} and~\ref{algo:dnn_those_active} suffer from this issue since they place unprecedented emphasis on the low density region. We then illustrate this effect by a logistic linear classification example in the following section.   

\begin{figure}[t]
\centering
\begin{minipage}[t]{0.33\textwidth}
\centering
\includegraphics[width=5cm]{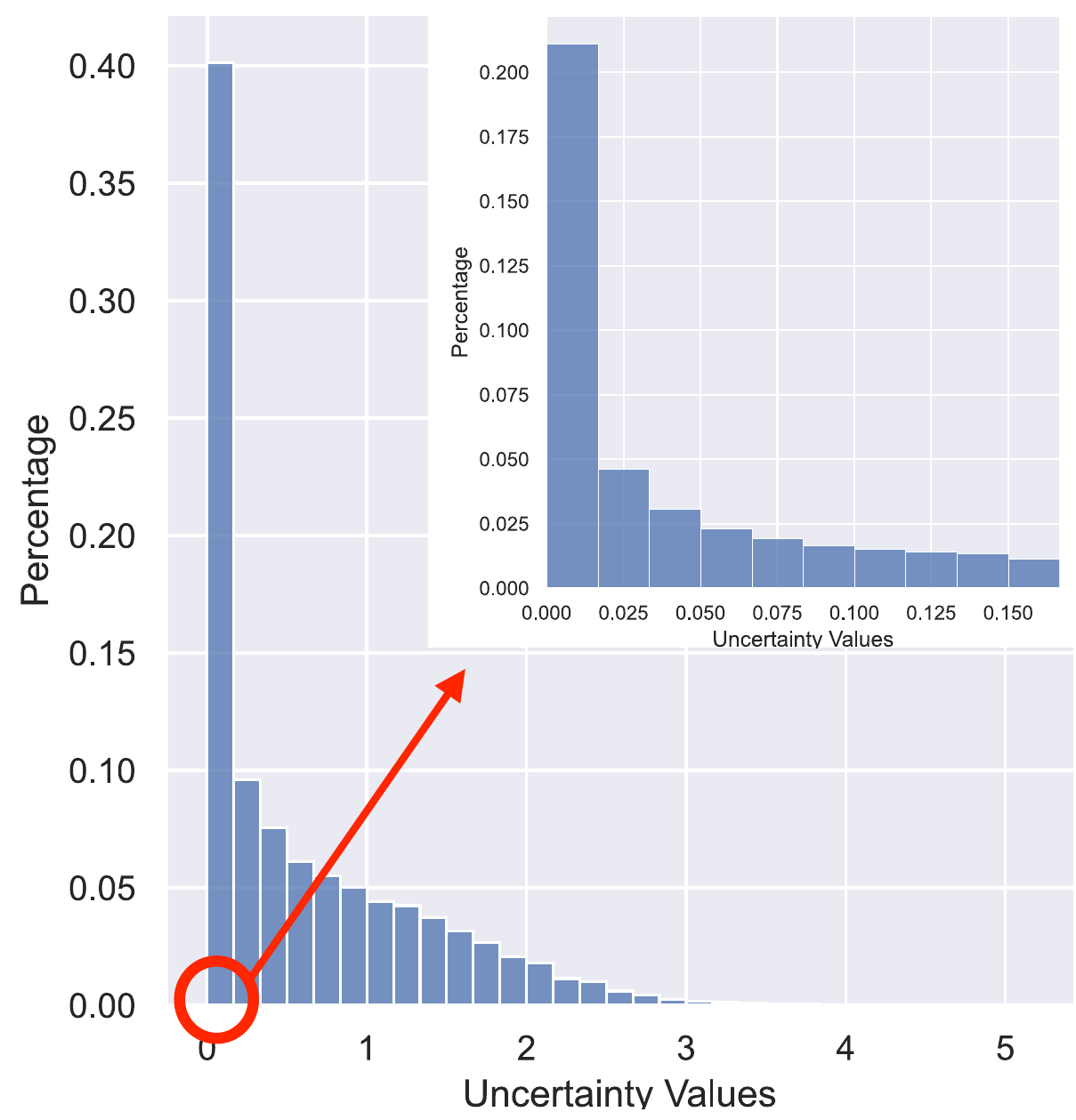}
\caption*{(a)The full dataset.}
\end{minipage}
\begin{minipage}[t]{0.3\textwidth}
\centering
\includegraphics[width=5.25cm]{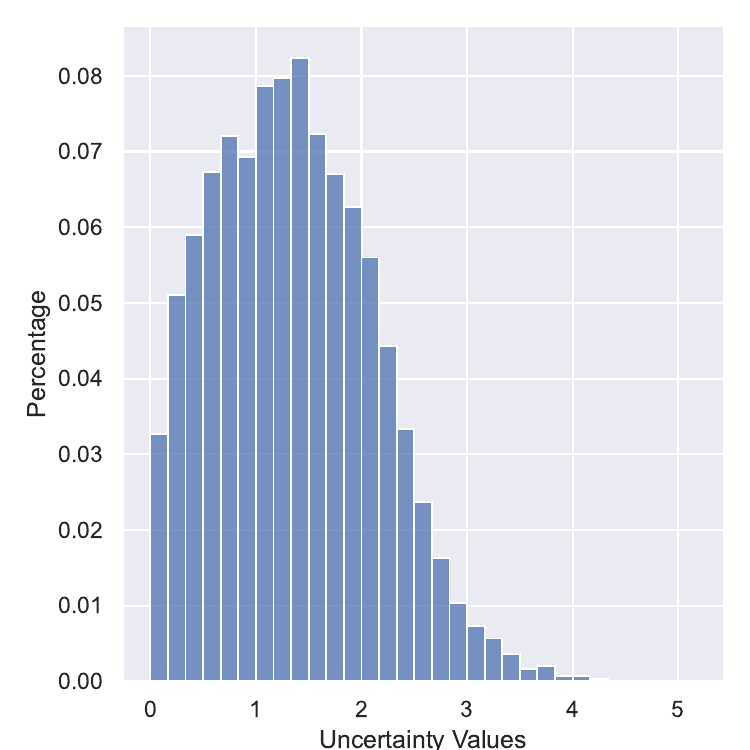}
\caption*{(b) The 3000 samples selected by COPS-vanilla. }
\end{minipage}
\begin{minipage}[t]{0.3\textwidth}
\centering
\includegraphics[width=5.25cm]{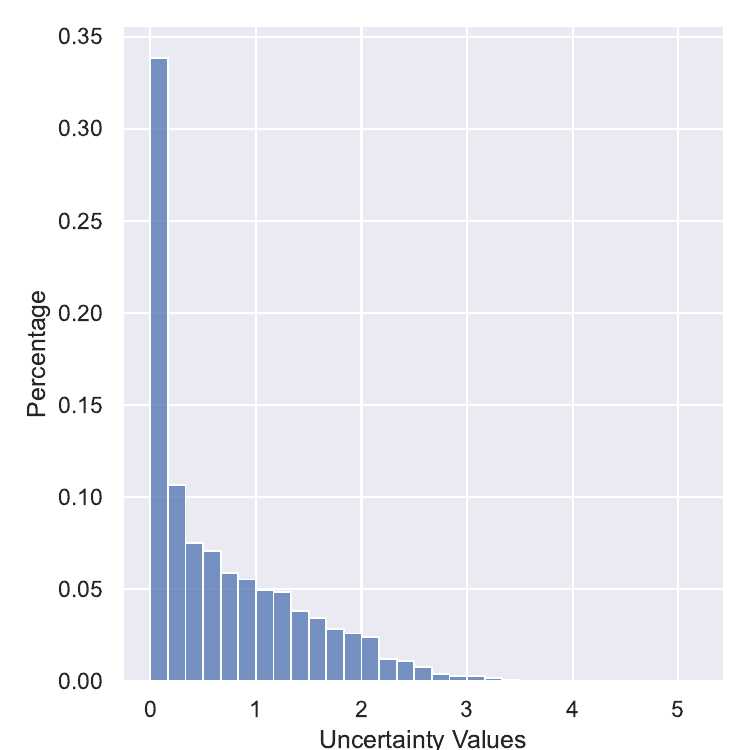}
\caption*{(c) The 3000 samples selected by COPS-clip. }
\end{minipage}
\caption{Histogram of estimated uncertainty of samples on CIFAR10-3000-WithY.}

\label{fig:cifar10_withy_uncertainty_histogram}
\end{figure}

\begin{figure}[t]
\centering
\begin{minipage}[t]{0.35\textwidth}
\centering
\includegraphics[width=5cm]{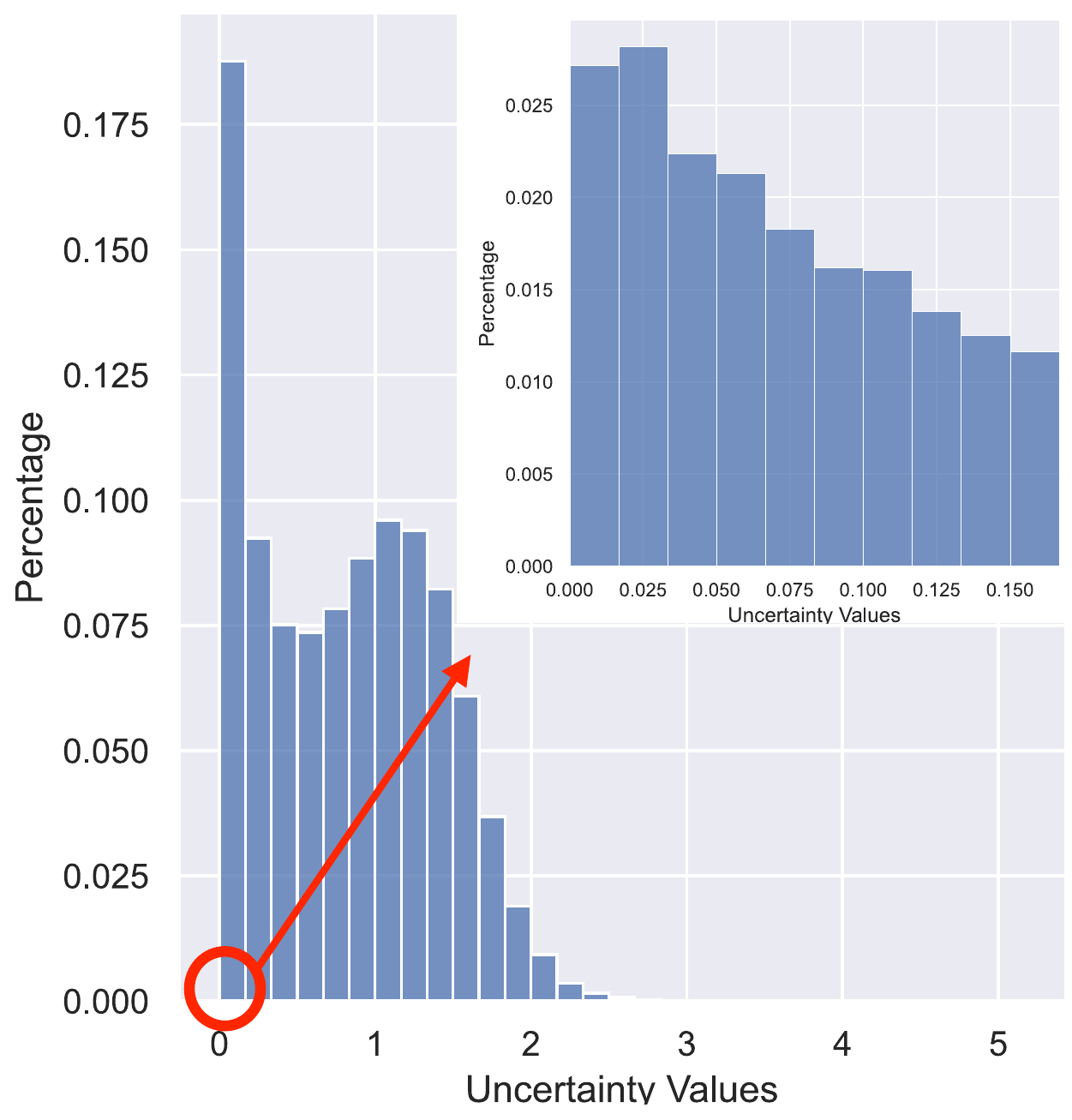}
\caption*{(a) On the full sampling dataset.}
\end{minipage}
\begin{minipage}[t]{0.3\textwidth}
\centering
\includegraphics[width=5.25cm]{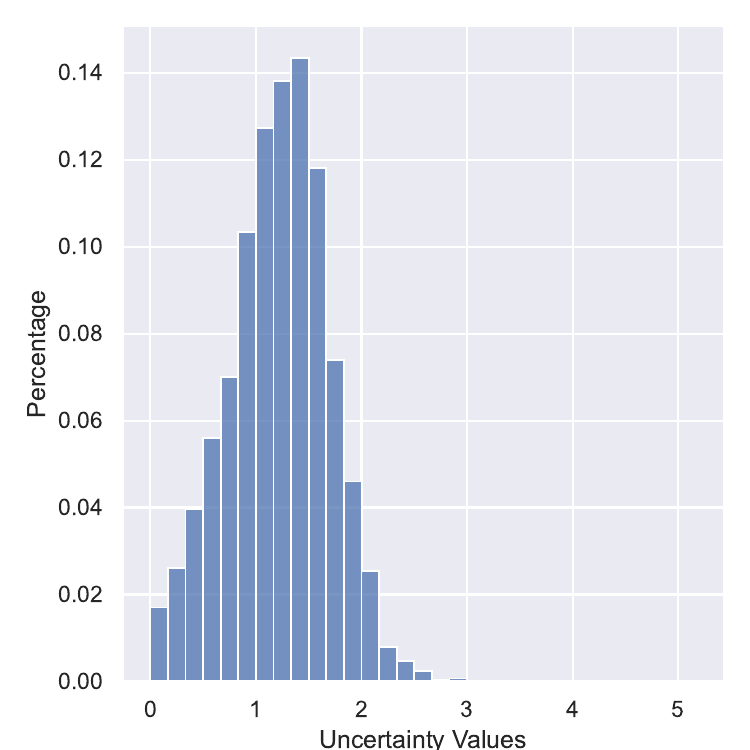}
\caption*{(b) Subset selected by COPS-vanilla}
\end{minipage}
\begin{minipage}[t]{0.3\textwidth}
\centering
\includegraphics[width=5.25cm]{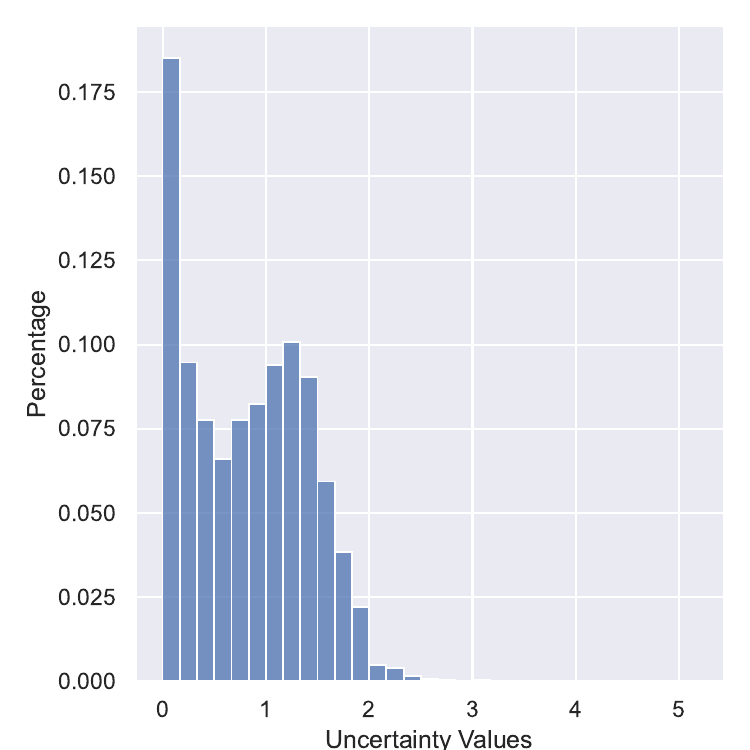}
\caption*{(c) Subset selected by COPS-clip}
\end{minipage}
\caption{Histogram of estimated uncertainty on CIFAR10 without labels (active learning).}

\label{fig:cifar10_withoy_uncertainty_histogram}
\end{figure}

\subsection{Simulating the effect of model misspecification on sampling algorithms}

\paragraph{Simulation with a linear example.} The optimal sampling strategy Eqn.~\eqref{eqn:withy_optimal_sampling_multiclass}~and~\eqref{eqn:withouty_optimal_sampling_multiclass} is derived under the assumption that the model is well-specified, i.e., there exists an oracle $\beta^*$ such that  $\bbP(y=c_k|\bmx) = p_k(\beta^*; \bmx)$ for all $\bmx$ and $k$. To illustrate how the uncertainty sampling can suffer from model misspecification, we conduct simulations on the following example which contains model misspecification following the setting of \cite{he2022nearly, ye2023corruption, bogunovic2021stochastic}.

Consider a binary classification problem $y \in \{0, 1\}$ with 2-dimensional input $\bx \in \bbR^2$. The true parameter $\beta^* = [2, 2]^\top$. In this simulation, we consider adversarial corruption, a typical case of misspecification in a line of previous research \cite{he2022nearly, ye2023corruption, bogunovic2021stochastic}. In this case, an adversary corrupts the classification responses $\by$ before they are revealed to the learners. Hence, if the learner still make estimations via the linear logistic model, the misspecification occurs. Suppose that the there exists model misspecification characterized by $\zeta: \bbR^2 \xrightarrow[]{} \bbR$ such that 
\begin{align}
    \label{eqn:corrupt_y}
     P(y=1|\bx, \beta^*, \zeta) = \frac{\exp{(\bx^\top \beta^* + \zeta(x))}}{1+\exp{(\bx^\top \beta^* + \zeta(x))}}.
\end{align}
Consider a training dataset consisting of 1,000 instances of $\bmx_1$, 100,000 instances of $\bmx_2$, and 100,000 instances of $\bmx_3$, where $\bmx_1=[1, 0]$, $\bmx_2=[0.1, 0.1]$, and $\bmx_3=[0, 1]$. It is evident that $\bmx_1$ falls within the low density region. In the following part, we will introduce non-zero corruption on $\bmx_1$. It is easy to infer that a corruption on $\bmx_1$ would induce estimation error on the first dimension of $\beta$. We incorporate $\bmx_2$ within the dataset to ensure that the estimation error on the first dimension would affect the estimation error on the second dimension.

\begin{table}[t]
    \centering
    \begin{tabular}{c|c|c|c}
    \toprule
       $\bx$ & $\bmx_1 = [1, 0]^\top$ & $\bmx_2 = [0.1, 0.1]^\top$ & $\bmx_3 = [0, 1]^\top$ \\
       \midrule
        Sampling Set & $n_1 = 1,000$ & $n_2 = 100,000$ & $n_3 = 100,000$\\
        Testing Set & $n_1 = 1,000$ & $n_2 = 100,000$ & $n_3 = 100,000$\\
         \bottomrule
    \end{tabular}
    \caption{A simple example with 2-dimensional input $\bx \in \bbR^2$ and binary output $y \in \{0, 1\}$. There are three kinds of inputs as shown in the table. Both the training (sampling) and testing set contains 1,000 $\bmx_1$, 100,000 $\bmx_2$ and 100,000 $\bmx_3$, respectively.}
    \label{tab:simulation_example}
\end{table}

\begin{figure}[t]
\centering
\begin{minipage}[t]{0.48\textwidth}
\centering
\includegraphics[width=8cm]{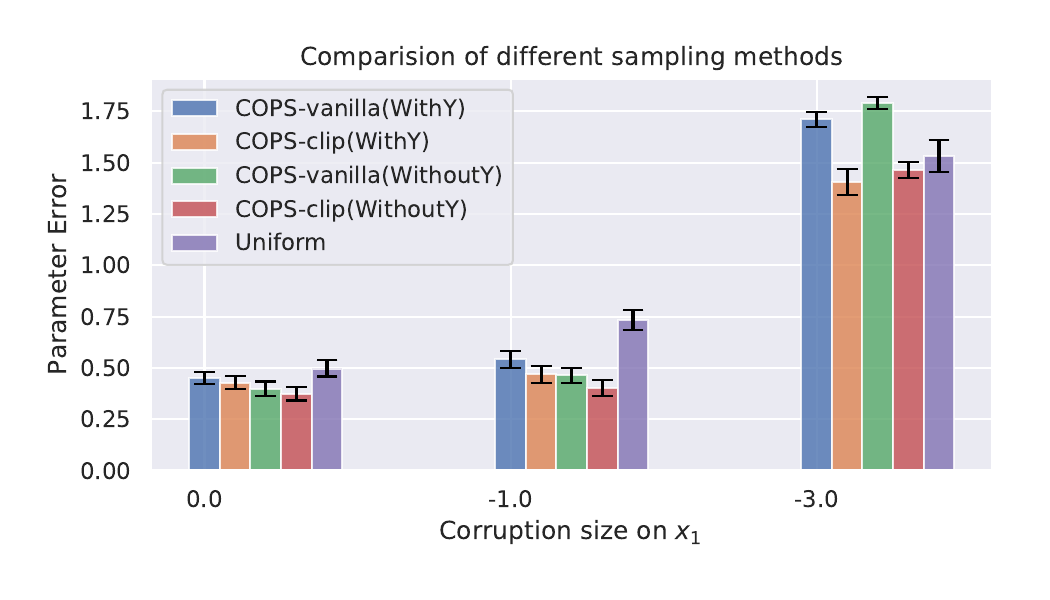}
\end{minipage}
\begin{minipage}[t]{0.48\textwidth}
\centering
\includegraphics[width=8cm]{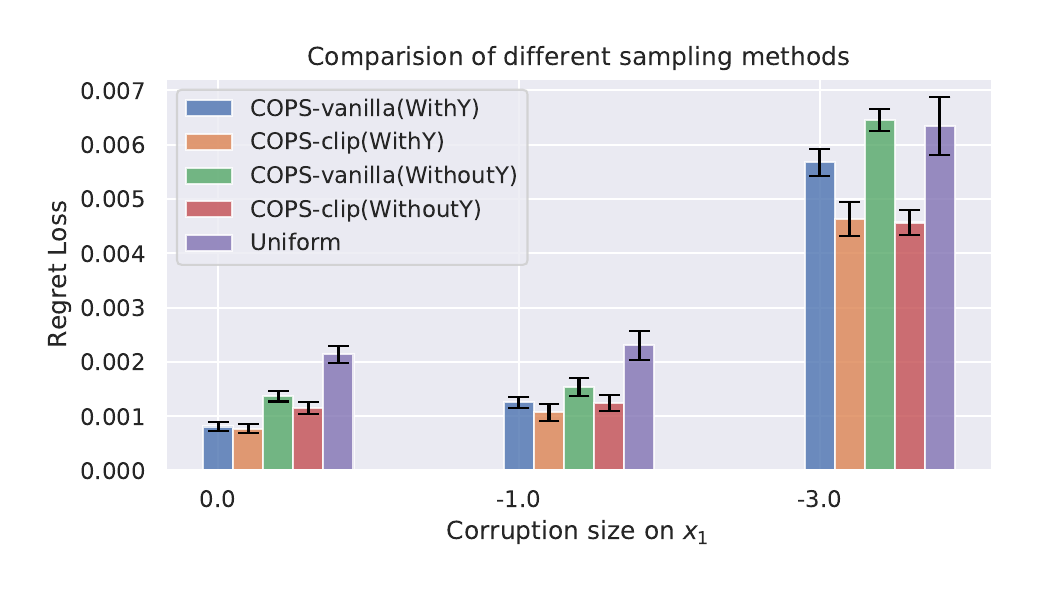}
\end{minipage}
\caption{Comparison on of different sampling methods on simulation data. Left)the error of parameter estimation $|\bar \beta - \beta^*|$ ; Right)  the regret loss $\cL(\bar \beta) - \cL(\beta^*)$ on the testing set. }
\label{fig:simulation_results}
\end{figure}

We conduct simulations involving three cases of corruption in the low-density region $\bmx_1$: (a) $\zeta(\bmx_1)=0$, (b) $\zeta(\bmx_1)=-1$, and (c) $\zeta(\bmx_1)=-3$.  We select 1,000 samples from a total of 201,000 samples and obtain $\bar \beta$ by uniform sampling, COPS for coreset selection (the linear Algorithm~\ref{algo:linear_those_coreset}) and COPS for active learning (the linear Algorithm~\ref{algo:linear_those_active}). We also visualize parameter estimation error $|\bar \beta - \beta^*|$. 
We evaluate the regret loss $\cL(\bar \beta) - \cL(\beta^*)$ on the testing set without corruption as shown in Table~\ref{tab:simulation_example}. The results of the comparison for each method are presented in 
Figure~\ref{fig:simulation_results}. The simulation results demonstrate that the vanilla uncertainty sampling (i.e., COPS-vanilla) strategy performs well when there is no corruption. However, as the level of corruption increases, the performance of uncertainty sampling deteriorates quickly and can be even worse than random sampling when $\zeta(\bmx_1)=-3$. 

\subsection{A simple fix} 
\cite{he2022nearly, ye2023corruption} argues that the corruption in the low density region can make $\beta_\MLE$ deviates significantly from $\beta^*$. To alleviate this problem, \cite{he2022nearly, ye2023corruption} propose to assign a smaller weight to the samples in low density regions when performing weighted linear regression, resulting in a solution closer to $\beta^*$. Specifically, they assign a weight $1/\max (\alpha, \|\bmx\|_{M_X^{-1}})$  to each sample to perform linear regression where $\alpha$ is a pre-defined hyper-parameter. For the samples with large uncertainty, they will have a small weight. 

Recall that we select data according to the uncertainty $u(\bmx, \bmy) = |\delta_1(\bmy) - p_1(\hat \beta_\MLE; \bmx)|\cdot \|\bmx\|_{M_X^{-1}}$. We can incorporate the idea of \cite{he2022nearly, ye2023corruption} through modifying the uncertainty sampling ratio by multiplying $u(\bmx, \bmy)$ with  $1/\max (\alpha, \|\bmx\|_{M_X^{-1}})$, i.e., draw samples according to $u(\bmx, \bmy)/\max (\alpha, \|\bmx\|_{M_X^{-1}}).$ Furthermore, since $u(\bmx, \bmy)$ and $\|\bmx\|_{M_X^{-1}}$ only differ by an scaling term $|\delta_1(\bmy) - p_1(\hat \beta_\MLE; \bmx)|$, we use an even simpler version $u(\bmx, \bmy)/\max (\alpha, u(\bmx, \bmy) ) \propto \min (\alpha,  u(\bmx, \bmy) )$, which turns out to simply threshold the maximum value of $u(\bmx, \bmy)$ for sampling. Therefore, the overall sampling ratio for coreset selection in Eqn~\eqref{eqn:optimal_ratio_binary_coreset} is modified as follows:
\begin{align}
        \label{eqn:optimal_ratio_binary_coreset_clip}
        \pi^\alpha(\bmx, \bmy) = \frac{\min(\alpha, u(\bmx, \bmy))}{\sum_{(\bmx', \bmy') \in \cS} \min(\alpha, u(\bmx', \bmy') \}},
    \end{align}
where $u(\bmx, \bmy) = |\delta_1(y) - p_1(\hat \beta_\MLE; \bmx')|\cdot\|\bmx'\|_{M_X^{-1}}$. The full modified algorithm for coreset selection is included in Algorithm~\ref{algo:dnn_those_coreset_clipping} in Appendix \ref{app:algorithms_with_clip}. The algorithm for active learning selection is also modified accordingly as shown in Algorithm~\ref{algo:dnn_those_active_clipping} in Appendix~\ref{app:algorithms_with_clip}. 
Notably, we don't modify the reweighting accordingly. Intuitively, original COPS  select samples by $u$ and the minimize the loss weighted by $1/u$. Here we select samples according to $\min\{\alpha, u\}$ but still use the original reweighting $1/u$. By this method, we can reduce the negative impact of model misspecification on the samples from the low density region i.e., samples with high uncertainty, obtaining a $\bar \beta$ closer to $\beta^*$.

We applied this method in the simulation experiment, testing the threshold at 3 or 10 times the minimum uncertainty. Take the threshold 3 for coreset selection for example, we set $\alpha = 3 \cdot\min_{(\bmx, \bmy) \in \cS}  u(\bmx, \bmy)$.   To differentiate, we use the suffix `COPS-clip' to represent the method with limited uncertainty from above. On the other hand, we refer to the unmodified COPS method as `COPS-vanilla'. The outcomes displayed in  Figure~\ref{fig:simulation_results} demonstrate how this straightforward approach enhances the performance of uncertainty sampling in case of substantial corruption, achieving significant improvement over both uniform sampling and COPS-vanilla in terms of both $|\bar \beta - \beta^*|$ and $\cL(\bar \beta) - \cL(\bar \beta^*)$.  The results in Figure~\ref{fig:cifar_vanilla}  show that the `COPS-clip' also works well in real world applications. 

Figure~\ref{fig:cifar10_withy_uncertainty_histogram}(c) and Figure~\ref{fig:cifar10_withoy_uncertainty_histogram}(c) illustrate the uncertainty distribution of the 3000 samples selected by COPS-clip in the CIFAR10-3000-WithY and CIFAR10-3000-WithoutY settings, respectively. We can see that compared to COPS-vanilla, COPS-clip selects samples whose uncertainty distribution is closer to the uncertainty distribution of the entire CIFAR10 dataset, with only a slight increase in samples exhibiting high uncertainty. In Appendix~\ref{sec:label_noise}, we provide additional results that COPS-vanilla selects a higher proportion of noisy data in CIFAR10-N compared to uniform sampling. However, COPS-clip does not exhibit an increase in the noisy ratio when compared to uniform sampling. 

\paragraph{Remark 1.} 
To simplify the discussion, let $u$ denote the $u(\bmx, \bmy)$ for coreset selection, and $u(\bmx)$ for active learning. In the vanilla COPS method, two stages are performed: 
(Stage 1): Data subsampling according to $u$.
(Stage 2): Weighted learning, where each selected sample is assigned a weight of $1/u$ to get an unbiased estimator. Since the sample weighting in Stage 2 involves calculating the inverse of $u$, it can result in high variance if $u$ approaches zero. To address this, previous work has implemented a threshold of $1/\max\{\beta, u\}$  \cite{ionides2008truncated, swaminathan2015counterfactual, citovsky2023leveraging}, which limits the minimum value of $u$. Both COPS-vanilla and COPS-clip adopt this strategy by default in the second stage to limit the variance and $\beta$ is set to 0.1 for all real-world dataset experiments (including the experiments in Figure~\ref{fig:cifar_vanilla}). Appendix~\ref{app:exp_details} shows the full details on this part. However, our empirical analysis reveals the importance of also limiting the maximum of $u$ by $\min\{\alpha, u\}$ in the first stage, which can alleviate the negative impact of potential model misspecification on COPS. To the best of our knowledge, this hasn't been discussed in existing works \cite{wang2018optimal, ting2018optimal, yao2019optimal, swaminathan2015counterfactual, citovsky2023leveraging}. Appendix~\ref{sec:thre_exp} presents empirical results to compare the impact of threshold on the first and second stages.



\section{Experiments and results}
\label{sect:main_experiments}
\paragraph{Settings.} In this section, we conduct extensive experiments to verify COPS. Here the COPS method refers to COPS-clip in Section~\ref{sect:thresholding_uncertainty} by default  and the detailed algorithms are in Algorithm~\ref{algo:dnn_those_coreset_clipping_withbeta}-\ref{algo:dnn_those_active_clipping_withbeta}. We compare COPS with various baseline methods, validate COPS on various datasets including both CV and NLP task and also datasets with natural label noise. For all the methods studied in this section, we use the same setting as described in Section~\ref{sect:thresholding_uncertainty} that we train probe networks on one probe dataset and performing sampling at once on the sampling dataset. 
The datasets used in our experiments are as follows:

\begin{itemize}
    \item CIFAR10~\cite{krizhevsky2009learning}: We utilize the original CIFAR10 dataset~\cite{krizhevsky2009learning}. To construct the probe set, we randomly select 1000 samples from each class, while the remaining training samples are used for the sampling set. For our experiments, we employ ResNet20, ResNet56~\cite{he2016deep}, MobileNetV2~\cite{sandler2018mobilenetv2}, and DenseNet121~\cite{huang2017densely} as our backbone models.

    \item CIFARBinary: We choose two classes, plane and car, from the CIFAR10 dataset for binary classification. Similar to CIFAR10, we assign 1000 samples from the training images for the probe set of each class, and the remaining training samples form the sampling set. In this case, we employ ResNet20 as our backbone model.

    \item CIFAR100: From the CIFAR100 dataset~\cite{krizhevsky2009learning}, we randomly select 200 samples for each class and assign them to the probe set. The remaining training samples are used in the sampling set. For this dataset, ResNet20 is utilized as the backbone model.

    \item CIFAR10-N: We use CIFAR10-N, a corrupted version of CIFAR10 introduced by Wei et al.~\cite{wei2021learning}. The training set of CIFAR10-N contains human-annotated real-world noisy labels collected from Amazon Mechanical Turk and the testing set of CIFAR10-N is the same with CIFAR10.  Similar to CIFAR10, we split 1000 samples from each class for the probe set, while the rest are included in the sampling set. We employ ResNet20 as our backbone model.

    \item IMDB: The IMDB dataset~\cite{maas-EtAl:2011:ACL-HLT2011} consists of positive and negative movie comments, comprising 25000 training samples and 25000 test samples. We split 5000 samples from the training set for uncertainty estimation and conduct our scheme on the remaining 20000 samples. For this dataset, we use a GRU-based structure~\cite{cho2014properties}, and further details can be found in Appendix~\ref{sec:archi}.

    \item SVHN: The SVHN dataset contains images of house numbers. We split 1000 samples from the train set for each class to estimate uncertainty, while the remaining train samples are used for the sampling schemes. ResNet20 serves as our backbone model in this case.

    \item Place365 (subset): We select ten classes from the Place365 dataset~\cite{zhou2017places}, each consisting of 5000 training samples and 100 testing samples. The chosen classes are department\_store, lighthouse, discotheque, museum-indoor, rock\_arch, tower, hunting\_lodge-outdoor, hayfield, arena-rodeo, and movie\_theater-indoor. We split the training set, assigning 1000 instances for each class to the probe set, and the remaining samples form the sampling set. ResNet18 is employed as the backbone model for this dataset.

\end{itemize}

We summarize the datasets in Table~\ref{tab:datasets_info}:
\begin{table}[H]
    \centering
    \resizebox{\linewidth}{!}{
    \begin{tabular}{c|c|c|c|c|c}
       \toprule
        Dataset & Class Number &Probe Set & Sampling Set & Target Size of Sub-sampling & Test Set  \\
        \midrule 
        CIFARBinary & 2 & 2,000 & 8,000& 600/2,000/6,000&2,000  \\
        CIFAR10 & 10&10,000&40,000&3,000/10,000/20,000&10,000\\
        CIFAR10-N & 10&10,000&40,000&3,000/10,000/20,000&10,000\\
        CIFAR100 & 100&20,000&30,000&3,000/10,000/20,000&10,000\\
        SVHN & 10&10,000&63,257&3,000/10,000/20,000&26,032\\
        Places365 & 10&10,000&40,000&3,000/10,000/20,000&1,000\\
        IMDB & 2 &5,000 &20,000&2,000/4,000/10,000&25,000\\
        \bottomrule
    \end{tabular}
    }
    \caption{The table provides descriptions of the datasets used in our study. The "Probe Set/ Sampling Set" column indicates the number of samples included in the Probe Set and Sampling Set for each dataset. The ``Target Size of Sub-sampling" column represents the number of samples selected from the Sampling Set for sub-sampling. For example, if the value is shown as "600", it indicates that we choose 600 instances from the Sampling Set for sub-sampling.}
    \label{tab:datasets_info}
\end{table}

\paragraph{Comparison with Baselines.} In this part, we compare our method COPS  with existing sample selection methods. We adopt competitive baselines for coreset selection and active learning, respectively.
The baselines for coreset selection (WithY) are as follows:
\begin{itemize}
    \item \textbf{Uniform sampling}.
    \item \textbf{IWeS(WithY)} \cite{citovsky2023leveraging} first fit two functions $f_{\theta^{(1)}}$ and $f_{\theta^{(2)}}$ on the probe set and then use the disagreement of the two functions with respect to entropy to calculate the sampling ratio for a sample $(\bmx, \bmy)$: 
    \begin{equation}
        \pi(\bmx, \bmy)= \sum_{k=0}^K\delta_k(\bmy)\big| p_k(f_{\theta^{(1)}}; \bmx)\log_2(p_k(f_{\theta^{(1)}}; \bmx)) - p_k(f_{\theta^{(2)}}; \bmx)\log_2(p_k(f_{\theta^{(2)}}; \bmx))\big|
    \end{equation}
    \item \textbf{BADGE(WithY)}\cite{ash2019deep}  calculates the gradient of the last layer and use kmeans++ to cluster the gradient. They then select the samples closest to cluster centers.
    \item \textbf{Margin} \cite{scheffer2001active}. The margin is computed by subtracting the predicted probability of the true class from 1. 
    \begin{align}
            \pi(\bmx, \bmy) = 1 - \sum_{k=1}^K \delta_k(\bmy) p_k(f_\theta; \bmx)
    \end{align}
\end{itemize}

\begin{figure}[t]
\centering
\begin{minipage}[t]{0.32\linewidth}
\centering
\includegraphics[width=\linewidth]{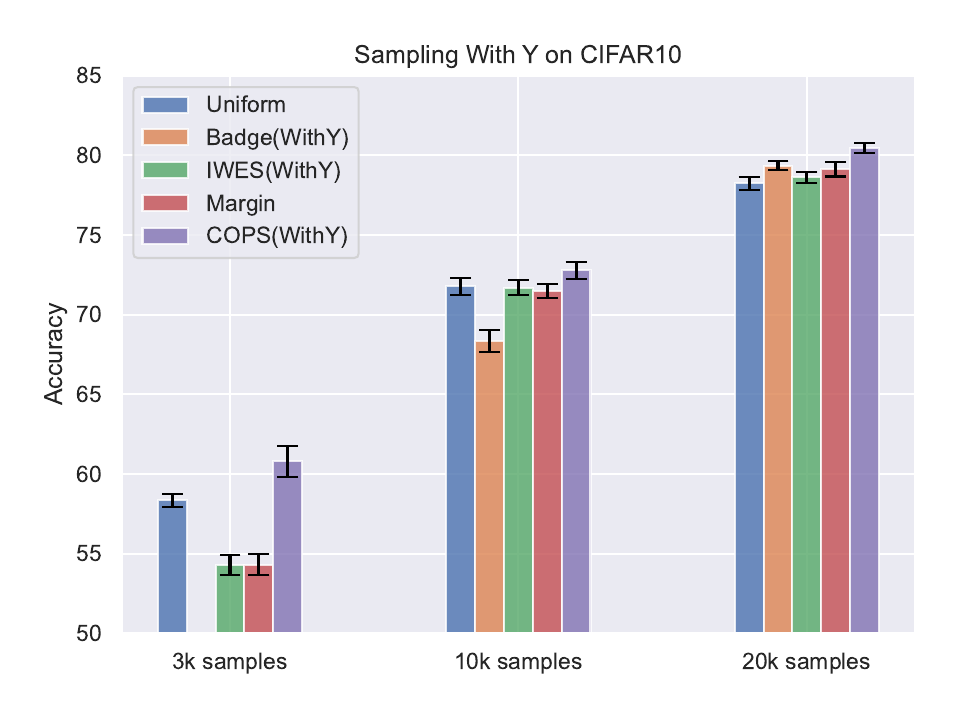}
\caption*{(a) CIFAR10.}
\end{minipage}
\begin{minipage}[t]{0.32\linewidth}
\centering
\includegraphics[width=\linewidth]{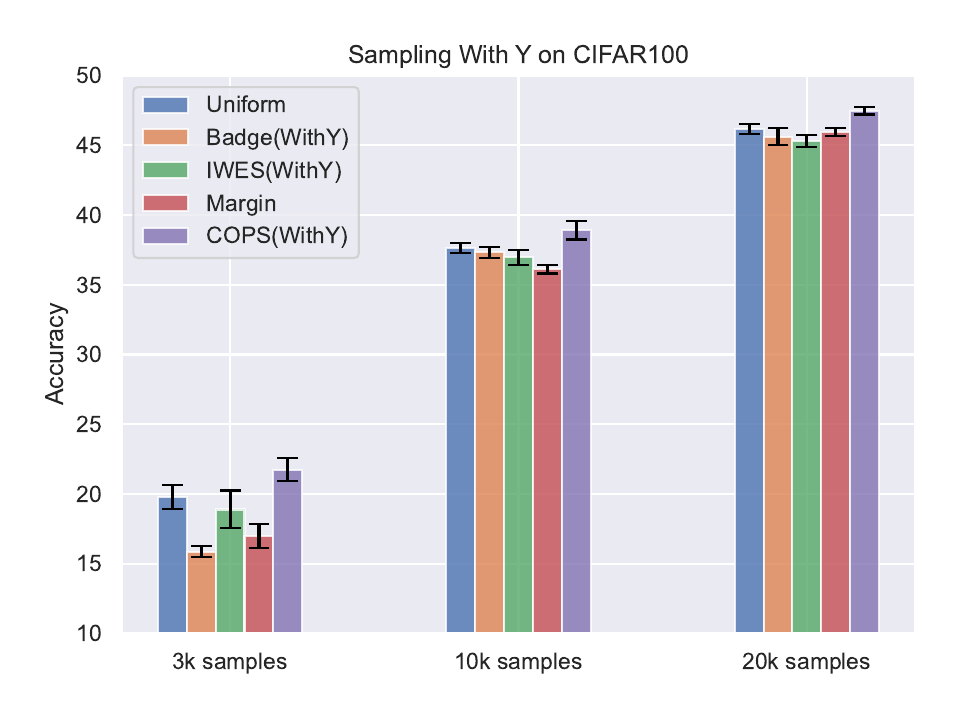}
\caption*{(c)CIFAR100}
\end{minipage}
\begin{minipage}[t]{0.32\linewidth}
\centering
\includegraphics[width=\linewidth]{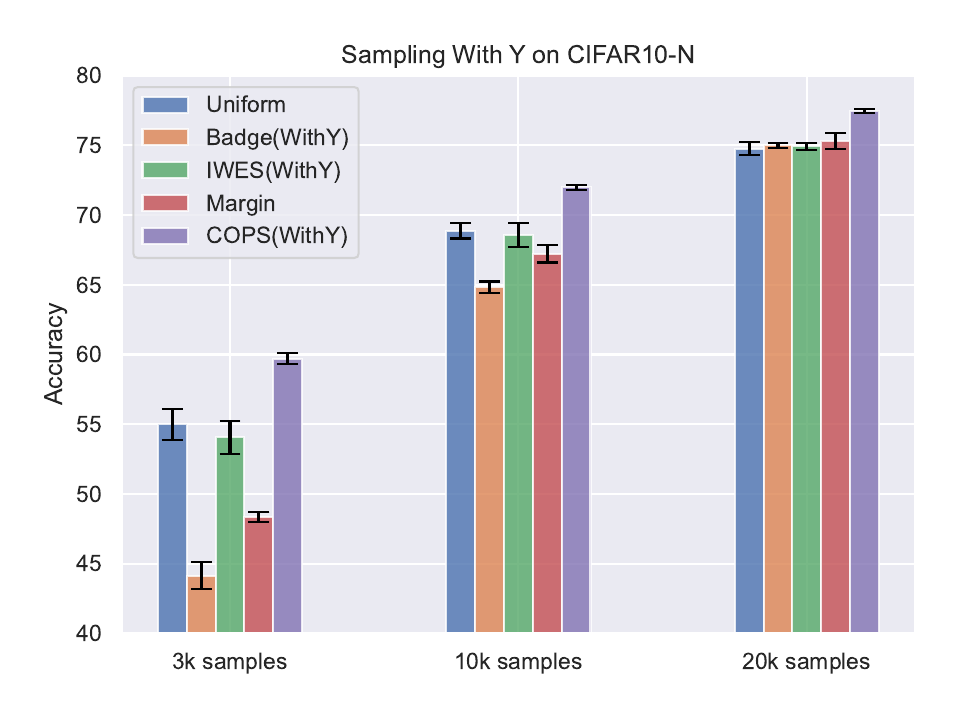}
\caption*{(b)CIFAR10-N}
\end{minipage}
\caption{Results for coreset selection (WithY). \label{fig:three_datsets_wy} For Badge with 3000 samples, the performance is lower than 50, so the bar is clipped in our figures.}
\end{figure}

\begin{figure}[t]
\centering
\begin{minipage}[t]{0.32\linewidth}
\centering
\includegraphics[width=\linewidth]{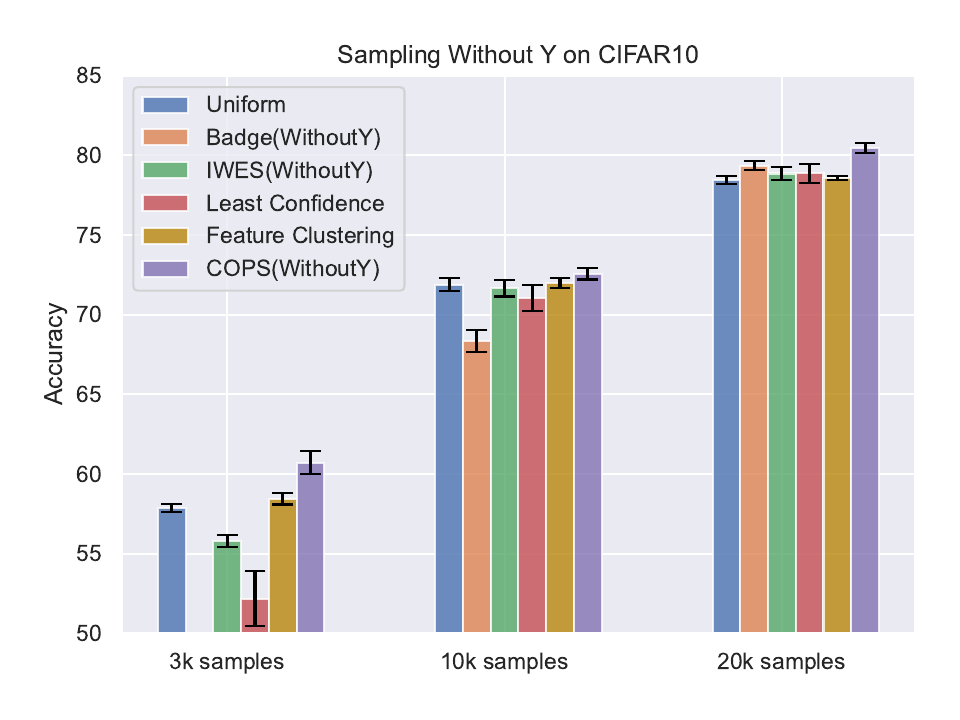}
\caption*{(a) CIFAR10.}
\end{minipage}
\begin{minipage}[t]{0.32\linewidth}
\centering
\includegraphics[width=\linewidth]{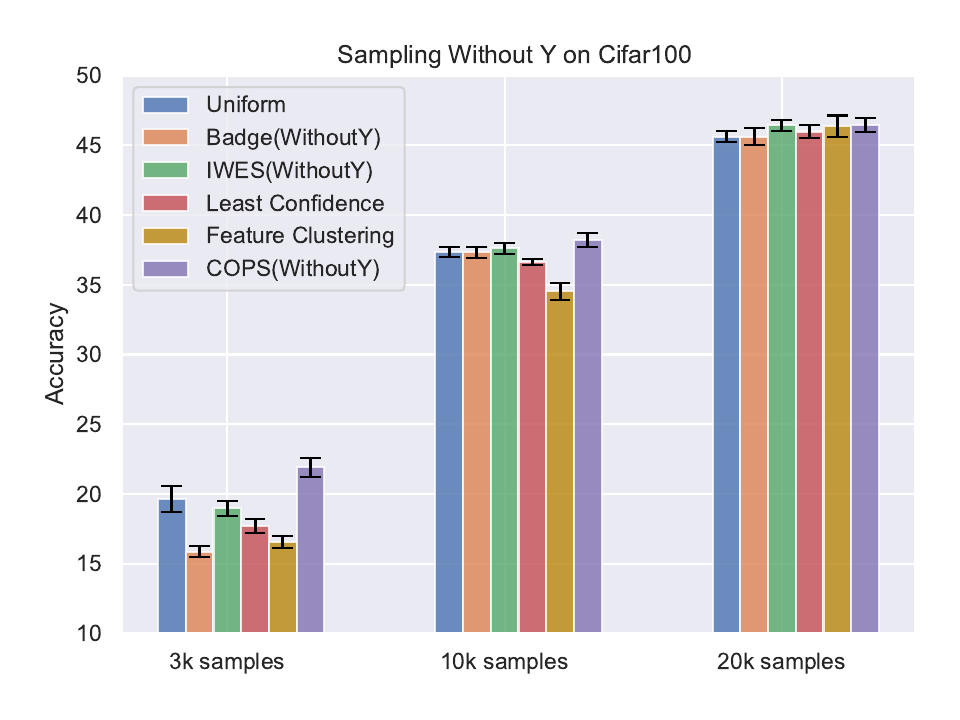}
\caption*{(b)CIFAR100}
\end{minipage}
\begin{minipage}[t]{0.32\linewidth}
\centering
\includegraphics[width=\linewidth]{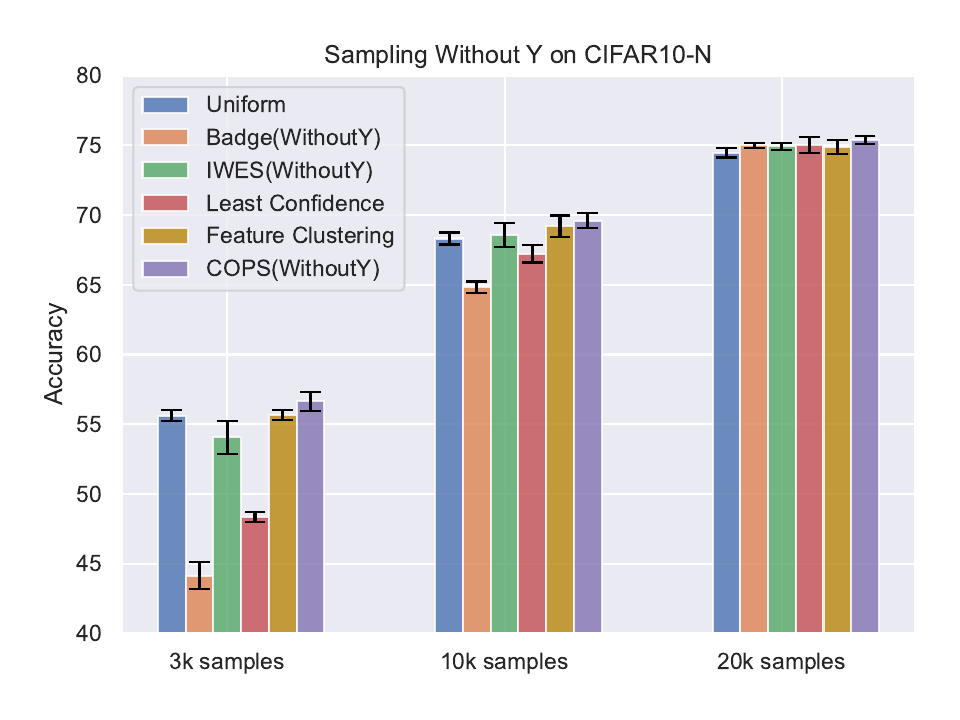}
\caption*{(c) CIFAR10-N}
\end{minipage}
\caption{Results for active learning (WithoutY). \label{fig:three_datsets_woy}For Badge with 3000 samples, the performance is lower than 50, so the bar is clipped in our figures.}
\end{figure}

\begin{figure}[t]
\centering
\begin{minipage}[t]{0.48\textwidth}
\centering
\includegraphics[width=6cm]{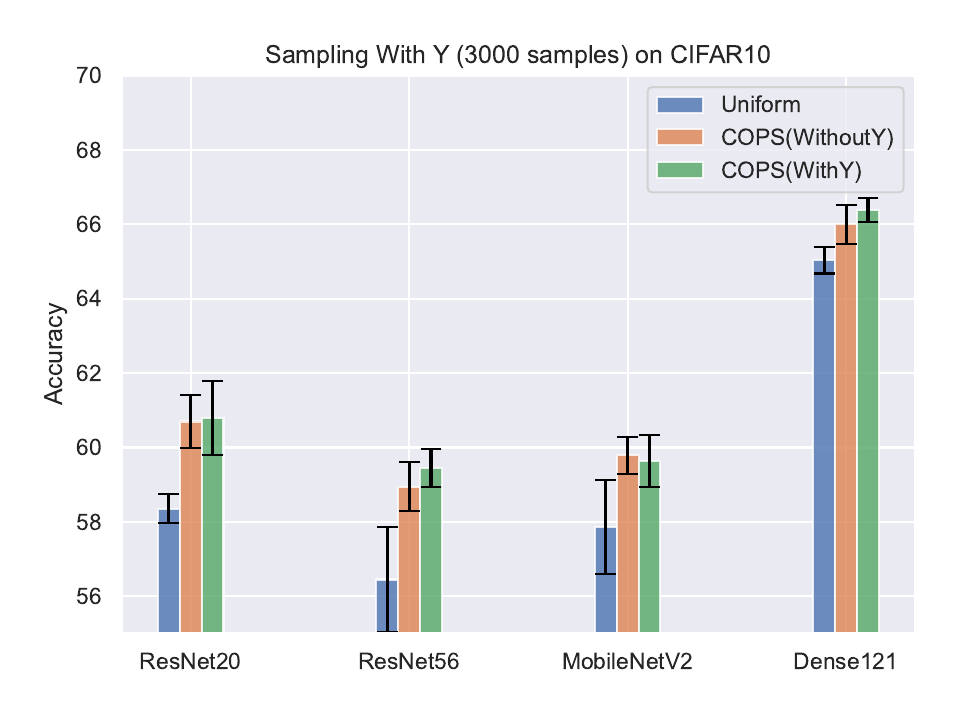}
\caption*{(a) CIFAR10-3000}
\end{minipage}
\begin{minipage}[t]{0.48\textwidth}
\centering
\includegraphics[width=6cm]{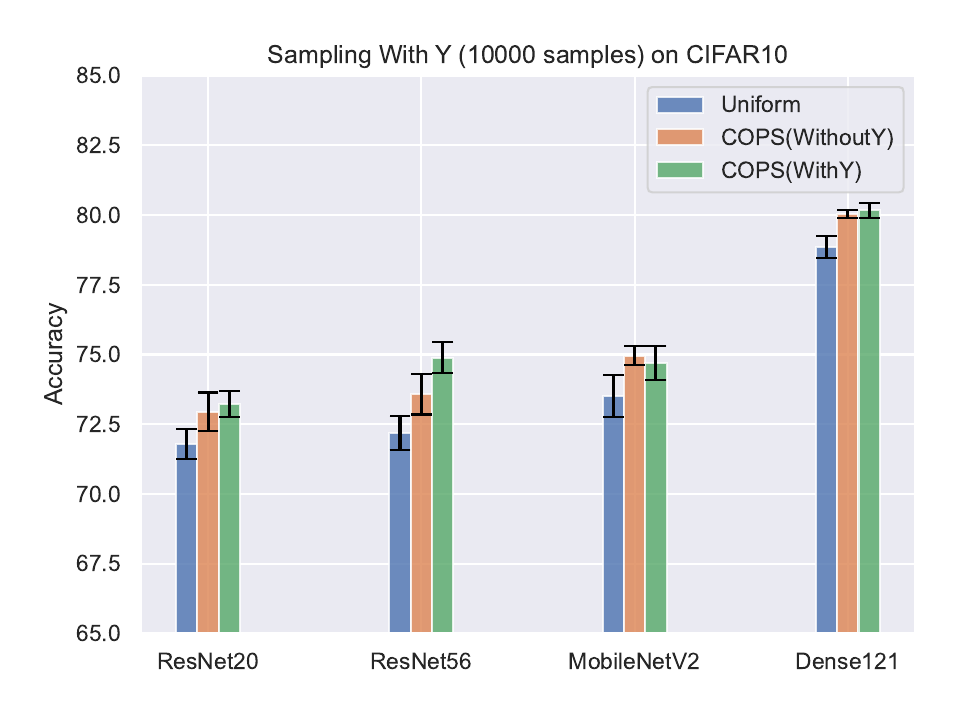}
\caption*{(b)CIFAR10-10000 }
\end{minipage}
\caption{Results for Cifar10 with different architectures.\label{cifar10_archi}}
\end{figure}

The baselines for active learning selection (WithoutY) are as follows:
\begin{itemize}
    \item \textbf{Uniform sampling}.
    \item \textbf{IWeS (WithoutY) }\cite{citovsky2023leveraging}  uses a normalized version of entropy is as follows, 
    \begin{equation}
         \pi(\bmx) = -\sum_{k = 0}^K p_k(f_\theta; \bmx)\log_2(p_k(f_\theta; \bmx)) /\log_2(K)
    \end{equation}
    \item \textbf{BADGE (WithoutY)}\cite{ash2019deep} 
   first obtain the pseudo label $\hat \bmy = \argmax_k p_k(f_\theta; \bmx)$ and the calculates the gradient of the last layer with the pseudo label $\hat \bmy$. Then they use K-means++ to cluster samples and select the samples closest to cluster centers. 

    \item  \textbf{Least confidence} \cite{scheffer2001active} is determined by calculating the difference between 1 and the highest probability assigned to a class:
        \begin{align}
            \pi(\bmx) = 1 - \max_{k} p_k(f_\theta; \bmx)
        \end{align}

    \item \textbf{Feature Clustering} \cite{sener2017active}\footnote{\cite{sener2017active} named their method as coreset, whereas, we refer to their method as feature clustering in order to avoid confusion with the coreset task.} first latent feature of the model and then uses K-means cluster the samples by its feature. They further select the samples closest to cluster.
\end{itemize}
We first compare COPS with the above baselines on both coreset selection (WithY) and active learning (WithoutY) settings on three datasets, CIFAR10, CIFAR100 and CIFAR10-N. The results in Figure~\ref{fig:three_datsets_wy} and  \ref{fig:three_datsets_woy} show that COPS can consistently outperform the baselines in these settings. The improvement is even more significant on CIFAR10-N, which contains nature label noise.

\paragraph{Multiple Architectures.} To verify the effectiveness of COPS, we conduct experiments on CIFAR10 with different neural network structures. Specifically, we choose several widely-used structures, including ResNet56 \cite{he2016deep}, MobileNetV2 \cite{sandler2018mobilenetv2} and DenseNet121 \cite{huang2017densely}. The results are shown in Figure.~\ref{cifar10_archi}. Our method COPS can stably improve over random sampling for both WithY and WithoutY on different DNN architectures.

\paragraph{Additional Datasets.} Furthermore, we evaluate the effectiveness of COPS on three additional datasets: SVHN, Places365 (subset), and IMDB (an NLP dataset). The results in Fig.~\ref{datasets_new} consistently demonstrate that our method consistently outperforms random sampling on these datasets.



 

\paragraph{Different methods for uncertainty estimation.}

In Algorithm~\ref{algo:uncertainty_estimation_DNN}, we obtain $M$ models $\{f_{\theta^{(m)}}\}_{m=1}^M$ on the probe dataset $\cS'$ by training DNNs independently with different initializations and random seeds. This method is referred to as the \textbf{different initialization} method. In this section, we compare this method with two alternative approaches to obtain $\{f_{\theta^{(m)}}\}_{m=1}^M$ given $\cS'$:

\begin{itemize}
\item [(a)] \textbf{Bootstrap}: Each $f_{\theta^{(m)}}$ is obtained by training a DNN on a randomly drawn subset from $\cS'$.
\item [(b)] \textbf{Dropout}~\cite{gal2016dropout}: A single DNN is trained on $\cS'$ with dropout. Then, $\{f_{\theta^{(m)}}\}_{m=1}^M$ are obtained by performing Monte Carlo Dropout during inference for $M$ iterations.
\end{itemize}

The comparison of these three methods on CIFAR10 is depicted in Figure~\ref{different_un}. It is evident that the different initialization method achieves the best performance, while the bootstrap method performs the worst among the three. The dropout method shows similar performance to the different initialization method in the coreset selection task (WithY).

\begin{figure}[t]
\centering
\begin{minipage}[t]{0.32\textwidth}
\centering
\includegraphics[width=\linewidth]{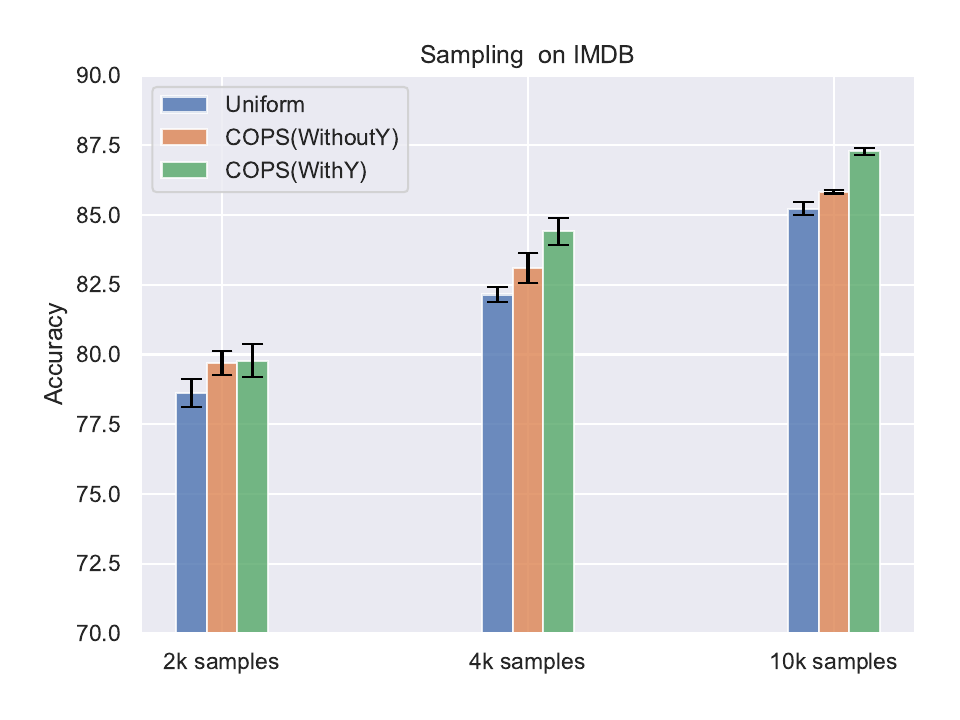}
\caption*{(a)IMDB.}
\end{minipage}
\begin{minipage}[t]{0.32\textwidth}
\centering
\includegraphics[width=\linewidth]{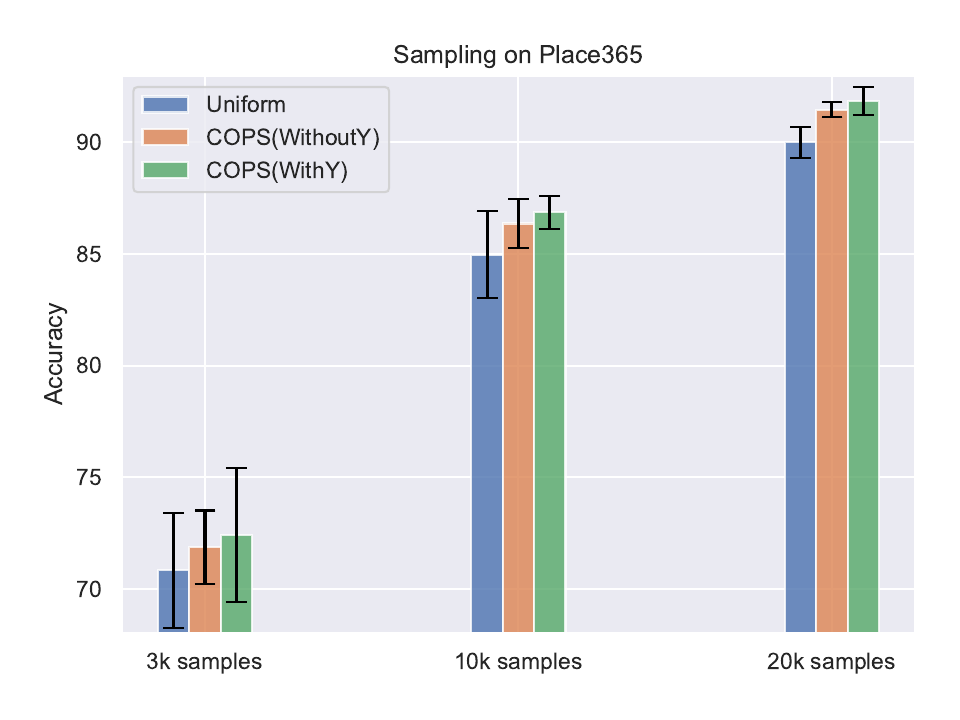}
\caption*{(b) Places365.}
\end{minipage}
\begin{minipage}[t]{0.32\textwidth}
\centering
\includegraphics[width=\linewidth]{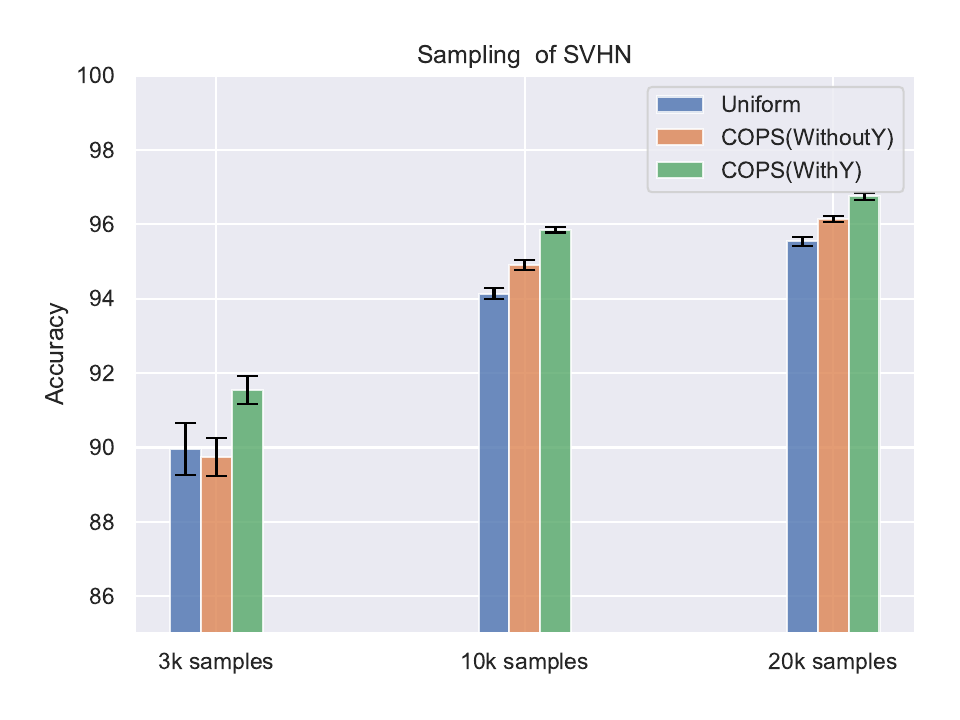}
\caption*{(c) SVHN.}
\end{minipage}
\caption{\label{datasets_new}Results of COPS on IMDB, PLACE365 and SVHN.}
\end{figure}

\begin{figure}[t]
\centering
\begin{minipage}[t]{0.48\textwidth}
\centering
\includegraphics[width=6cm]{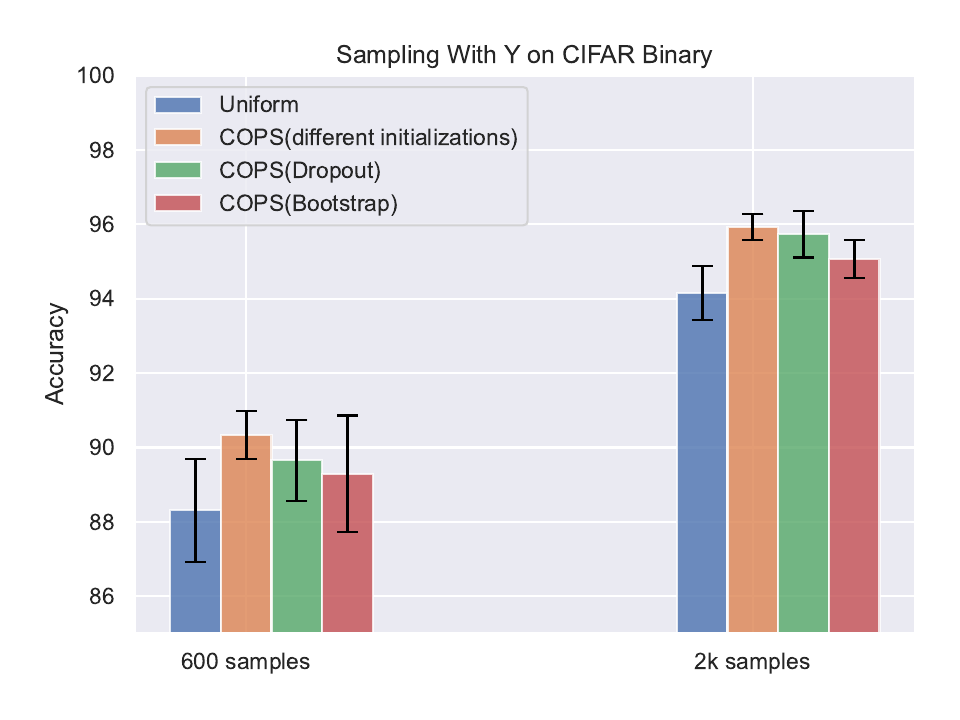}
\caption*{(a) WithY.}
\end{minipage}
\begin{minipage}[t]{0.48\textwidth}
\centering
\includegraphics[width=6cm]{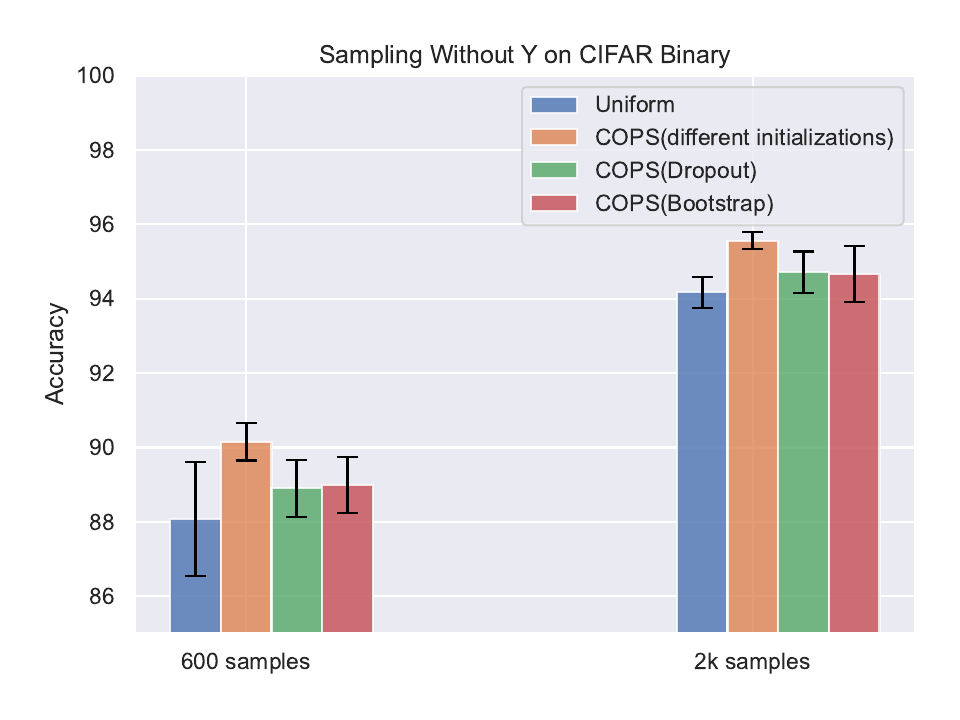}
\caption*{(b)WithoutY.}
\end{minipage}
\caption{Results for different kinds of uncertainty on CIFAR10-8,000 .\label{different_un}}
\end{figure}

\section{Conclusion}
This study presents the COPS method, which offers a theoretically optimal solution for coreset selection and active learning in linear softmax regression. By leveraging the output of the models, the sampling ratio of COPS can be effectively estimated even in deep learning contexts. To address the challenge of model sensitivity to misspecification, we introduce a downweighting approach for low-density samples. By incorporating this strategy, we modify  the sampling ratio of COPS through thresholding the sampling ratio. Empirical experiments conducted on benchmark datasets, utilizing deep neural networks, further demonstrate the effectiveness of COPS in comparison to baseline methods. The results highlight the superiority of COPS in achieving optimal subsampling and performance improvement.

{\small

\bibliographystyle{plain}
}
\bibliography{reference}

\begin{thebibliography}{10}

\bibitem{ash2019deep}
Jordan~T Ash, Chicheng Zhang, Akshay Krishnamurthy, John Langford, and Alekh
  Agarwal.
\newblock Deep batch active learning by diverse, uncertain gradient lower
  bounds.
\newblock {\em arXiv preprint arXiv:1906.03671}, 2019.

\bibitem{bogunovic2021stochastic}
Ilija Bogunovic, Arpan Losalka, Andreas Krause, and Jonathan Scarlett.
\newblock Stochastic linear bandits robust to adversarial attacks.
\newblock In {\em International Conference on Artificial Intelligence and
  Statistics}, pages 991--999. PMLR, 2021.

\bibitem{borsos2020coresets}
Zal{\'a}n Borsos, Mojmir Mutny, and Andreas Krause.
\newblock Coresets via bilevel optimization for continual learning and
  streaming.
\newblock {\em Advances in neural information processing systems},
  33:14879--14890, 2020.

\bibitem{bubeck2012regret}
S{\'e}bastien Bubeck, Nicolo Cesa-Bianchi, et~al.
\newblock Regret analysis of stochastic and nonstochastic multi-armed bandit
  problems.
\newblock {\em Foundations and Trends{\textregistered} in Machine Learning},
  5(1):1--122, 2012.

\bibitem{cho2014properties}
Kyunghyun Cho, Bart Van~Merri{\"e}nboer, Dzmitry Bahdanau, and Yoshua Bengio.
\newblock On the properties of neural machine translation: Encoder-decoder
  approaches.
\newblock {\em arXiv preprint arXiv:1409.1259}, 2014.

\bibitem{citovsky2023leveraging}
Gui Citovsky, Giulia DeSalvo, Sanjiv Kumar, Srikumar Ramalingam, Afshin
  Rostamizadeh, and Yunjuan Wang.
\newblock Leveraging importance weights in subset selection.
\newblock {\em arXiv preprint arXiv:2301.12052}, 2023.

\bibitem{clarkson2016fast}
Kenneth~L Clarkson, Petros Drineas, Malik Magdon-Ismail, Michael~W Mahoney,
  Xiangrui Meng, and David~P Woodruff.
\newblock The fast cauchy transform and faster robust linear regression.
\newblock {\em SIAM Journal on Computing}, 45(3):763--810, 2016.

\bibitem{culotta2005reducing}
Aron Culotta and Andrew McCallum.
\newblock Reducing labeling effort for structured prediction tasks.
\newblock In {\em AAAI}, volume~5, pages 746--751, 2005.

\bibitem{dean2018regret}
Sarah Dean, Horia Mania, Nikolai Matni, Benjamin Recht, and Stephen Tu.
\newblock Regret bounds for robust adaptive control of the linear quadratic
  regulator.
\newblock {\em Advances in Neural Information Processing Systems}, 31, 2018.

\bibitem{drineas2012fast}
Petros Drineas, Malik Magdon-Ismail, Michael~W Mahoney, and David~P Woodruff.
\newblock Fast approximation of matrix coherence and statistical leverage.
\newblock {\em The Journal of Machine Learning Research}, 13(1):3475--3506,
  2012.

\bibitem{drineas2011faster}
Petros Drineas, Michael~W Mahoney, Shan Muthukrishnan, and Tam{\'a}s
  Sarl{\'o}s.
\newblock Faster least squares approximation.
\newblock {\em Numerische mathematik}, 117(2):219--249, 2011.

\bibitem{feldman2011unified}
Dan Feldman and Michael Langberg.
\newblock A unified framework for approximating and clustering data.
\newblock In {\em Proceedings of the forty-third annual ACM symposium on Theory
  of computing}, pages 569--578, 2011.

\bibitem{gal2016dropout}
Yarin Gal and Zoubin Ghahramani.
\newblock Dropout as a bayesian approximation: Representing model uncertainty
  in deep learning.
\newblock In {\em international conference on machine learning}, pages
  1050--1059. PMLR, 2016.

\bibitem{gonccalves2005bootstrap}
S{\'\i}lvia Gon{\c{c}}alves and Halbert White.
\newblock Bootstrap standard error estimates for linear regression.
\newblock {\em Journal of the American Statistical Association},
  100(471):970--979, 2005.

\bibitem{har2005smaller}
Sariel Har-Peled and Akash Kushal.
\newblock Smaller coresets for k-median and k-means clustering.
\newblock In {\em Proceedings of the twenty-first annual symposium on
  Computational geometry}, pages 126--134, 2005.

\bibitem{he2022nearly}
Jiafan He, Dongruo Zhou, Tong Zhang, and Quanquan Gu.
\newblock Nearly optimal algorithms for linear contextual bandits with
  adversarial corruptions.
\newblock {\em arXiv preprint arXiv:2205.06811}, 2022.

\bibitem{he2016deep}
Kaiming He, Xiangyu Zhang, Shaoqing Ren, and Jian Sun.
\newblock Deep residual learning for image recognition.
\newblock In {\em Proceedings of the IEEE conference on computer vision and
  pattern recognition}, pages 770--778, 2016.

\bibitem{huang2017densely}
Gao Huang, Zhuang Liu, Laurens Van Der~Maaten, and Kilian~Q Weinberger.
\newblock Densely connected convolutional networks.
\newblock In {\em Proceedings of the IEEE conference on computer vision and
  pattern recognition}, pages 4700--4708, 2017.

\bibitem{huggins2016coresets}
Jonathan Huggins, Trevor Campbell, and Tamara Broderick.
\newblock Coresets for scalable bayesian logistic regression.
\newblock {\em Advances in neural information processing systems}, 29, 2016.

\bibitem{imberg2020optimal}
Henrik Imberg, Johan Jonasson, and Marina Axelson-Fisk.
\newblock Optimal sampling in unbiased active learning.
\newblock In {\em International Conference on Artificial Intelligence and
  Statistics}, pages 559--569. PMLR, 2020.

\bibitem{ionides2008truncated}
Edward~L Ionides.
\newblock Truncated importance sampling.
\newblock {\em Journal of Computational and Graphical Statistics},
  17(2):295--311, 2008.

\bibitem{Jacot2018NeuralTK}
Arthur Jacot, Franck Gabriel, and Cl{\'e}ment Hongler.
\newblock Neural tangent kernel: convergence and generalization in neural
  networks (invited paper).
\newblock {\em Proceedings of the 53rd Annual ACM SIGACT Symposium on Theory of
  Computing}, 2018.

\bibitem{jin2020provably}
Chi Jin, Zhuoran Yang, Zhaoran Wang, and Michael~I Jordan.
\newblock Provably efficient reinforcement learning with linear function
  approximation.
\newblock In {\em Conference on Learning Theory}, pages 2137--2143. PMLR, 2020.

\bibitem{joshi2009multi}
Ajay~J Joshi, Fatih Porikli, and Nikolaos Papanikolopoulos.
\newblock Multi-class active learning for image classification.
\newblock In {\em 2009 ieee conference on computer vision and pattern
  recognition}, pages 2372--2379. IEEE, 2009.

\bibitem{killamsetty2021grad}
Krishnateja Killamsetty, Sivasubramanian Durga, Ganesh Ramakrishnan, Abir De,
  and Rishabh Iyer.
\newblock Grad-match: Gradient matching based data subset selection for
  efficient deep model training.
\newblock In {\em International Conference on Machine Learning}, pages
  5464--5474. PMLR, 2021.

\bibitem{kim2021task}
Kwanyoung Kim, Dongwon Park, Kwang~In Kim, and Se~Young Chun.
\newblock Task-aware variational adversarial active learning.
\newblock In {\em Proceedings of the IEEE/CVF Conference on Computer Vision and
  Pattern Recognition}, pages 8166--8175, 2021.

\bibitem{koh2017understanding}
Pang~Wei Koh and Percy Liang.
\newblock Understanding black-box predictions via influence functions.
\newblock In {\em International conference on machine learning}, pages
  1885--1894. PMLR, 2017.

\bibitem{krizhevsky2009learning}
Alex Krizhevsky, Geoffrey Hinton, et~al.
\newblock Learning multiple layers of features from tiny images.
\newblock 2009.

\bibitem{lattimore2020bandit}
Tor Lattimore and Csaba Szepesv{\'a}ri.
\newblock {\em Bandit algorithms}.
\newblock Cambridge University Press, 2020.

\bibitem{loshchilov2018decoupled}
Ilya Loshchilov and Frank Hutter.
\newblock Decoupled weight decay regularization.
\newblock In {\em International Conference on Learning Representations}, 2019.

\bibitem{lucic2017training}
Mario Lucic, Matthew Faulkner, Andreas Krause, and Dan Feldman.
\newblock Training gaussian mixture models at scale via coresets.
\newblock {\em The Journal of Machine Learning Research}, 18(1):5885--5909,
  2017.

\bibitem{ma2014statistical}
Ping Ma, Michael Mahoney, and Bin Yu.
\newblock A statistical perspective on algorithmic leveraging.
\newblock In {\em International conference on machine learning}, pages 91--99.
  PMLR, 2014.

\bibitem{maas-EtAl:2011:ACL-HLT2011}
Andrew~L. Maas, Raymond~E. Daly, Peter~T. Pham, Dan Huang, Andrew~Y. Ng, and
  Christopher Potts.
\newblock Learning word vectors for sentiment analysis.
\newblock In {\em Proceedings of the 49th Annual Meeting of the Association for
  Computational Linguistics: Human Language Technologies}, pages 142--150,
  Portland, Oregon, USA, June 2011. Association for Computational Linguistics.

\bibitem{meng2014lsrn}
Xiangrui Meng, Michael~A Saunders, and Michael~W Mahoney.
\newblock Lsrn: A parallel iterative solver for strongly over-or
  underdetermined systems.
\newblock {\em SIAM Journal on Scientific Computing}, 36(2):C95--C118, 2014.

\bibitem{mirzasoleiman2020coresets}
Baharan Mirzasoleiman, Jeff Bilmes, and Jure Leskovec.
\newblock Coresets for data-efficient training of machine learning models.
\newblock In {\em International Conference on Machine Learning}, pages
  6950--6960. PMLR, 2020.

\bibitem{ren2021survey}
Pengzhen Ren, Yun Xiao, Xiaojun Chang, Po-Yao Huang, Zhihui Li, Brij~B Gupta,
  Xiaojiang Chen, and Xin Wang.
\newblock A survey of deep active learning.
\newblock {\em ACM computing surveys (CSUR)}, 54(9):1--40, 2021.

\bibitem{rice2006mathematical}
John~A Rice.
\newblock {\em Mathematical statistics and data analysis}.
\newblock Cengage Learning, 2006.

\bibitem{roth2006margin}
Dan Roth and Kevin Small.
\newblock Margin-based active learning for structured output spaces.
\newblock In {\em Machine Learning: ECML 2006: 17th European Conference on
  Machine Learning Berlin, Germany, September 18-22, 2006 Proceedings 17},
  pages 413--424. Springer, 2006.

\bibitem{sandler2018mobilenetv2}
Mark Sandler, Andrew Howard, Menglong Zhu, Andrey Zhmoginov, and Liang-Chieh
  Chen.
\newblock Mobilenetv2: Inverted residuals and linear bottlenecks.
\newblock In {\em Proceedings of the IEEE conference on computer vision and
  pattern recognition}, pages 4510--4520, 2018.

\bibitem{scheffer2001active}
Tobias Scheffer, Christian Decomain, and Stefan Wrobel.
\newblock Active hidden markov models for information extraction.
\newblock In {\em International symposium on intelligent data analysis}, pages
  309--318. Springer, 2001.

\bibitem{sener2017active}
Ozan Sener and Silvio Savarese.
\newblock Active learning for convolutional neural networks: A core-set
  approach.
\newblock {\em arXiv preprint arXiv:1708.00489}, 2017.

\bibitem{sinha2019variational}
Samarth Sinha, Sayna Ebrahimi, and Trevor Darrell.
\newblock Variational adversarial active learning.
\newblock In {\em Proceedings of the IEEE/CVF International Conference on
  Computer Vision}, pages 5972--5981, 2019.

\bibitem{swaminathan2015counterfactual}
Adith Swaminathan and Thorsten Joachims.
\newblock Counterfactual risk minimization: Learning from logged bandit
  feedback.
\newblock In {\em International Conference on Machine Learning}, pages
  814--823. PMLR, 2015.

\bibitem{ting2018optimal}
Daniel Ting and Eric Brochu.
\newblock Optimal subsampling with influence functions.
\newblock {\em Advances in neural information processing systems}, 31, 2018.

\bibitem{tropp2012user}
Joel~A Tropp.
\newblock User-friendly tail bounds for sums of random matrices.
\newblock {\em Foundations of computational mathematics}, 12:389--434, 2012.

\bibitem{tsang2005core}
Ivor~W Tsang, James~T Kwok, Pak-Ming Cheung, and Nello Cristianini.
\newblock Core vector machines: Fast svm training on very large data sets.
\newblock {\em Journal of Machine Learning Research}, 6(4), 2005.

\bibitem{wang2014new}
Dan Wang and Yi~Shang.
\newblock A new active labeling method for deep learning.
\newblock In {\em 2014 International joint conference on neural networks
  (IJCNN)}, pages 112--119. IEEE, 2014.

\bibitem{wang2018optimal}
HaiYing Wang, Rong Zhu, and Ping Ma.
\newblock Optimal subsampling for large sample logistic regression.
\newblock {\em Journal of the American Statistical Association},
  113(522):829--844, 2018.

\bibitem{wei2021learning}
Jiaheng Wei, Zhaowei Zhu, Hao Cheng, Tongliang Liu, Gang Niu, and Yang Liu.
\newblock Learning with noisy labels revisited: A study using real-world human
  annotations.
\newblock {\em arXiv preprint arXiv:2110.12088}, 2021.

\bibitem{wei2015submodularity}
Kai Wei, Rishabh Iyer, and Jeff Bilmes.
\newblock Submodularity in data subset selection and active learning.
\newblock In {\em International conference on machine learning}, pages
  1954--1963. PMLR, 2015.

\bibitem{yao2019optimal}
Yaqiong Yao and HaiYing Wang.
\newblock Optimal subsampling for softmax regression.
\newblock {\em Statistical Papers}, 60:585--599, 2019.

\bibitem{ye2023corruption}
Chenlu Ye, Wei Xiong, Quanquan Gu, and Tong Zhang.
\newblock Corruption-robust algorithms with uncertainty weighting for nonlinear
  contextual bandits and markov decision processes.
\newblock In {\em International Conference on Machine Learning}, pages
  39834--39863. PMLR, 2023.

\bibitem{yoo2019learning}
Donggeun Yoo and In~So Kweon.
\newblock Learning loss for active learning.
\newblock In {\em Proceedings of the IEEE/CVF conference on computer vision and
  pattern recognition}, pages 93--102, 2019.

\bibitem{zhou2017places}
Bolei Zhou, Agata Lapedriza, Aditya Khosla, Aude Oliva, and Antonio Torralba.
\newblock Places: A 10 million image database for scene recognition.
\newblock {\em IEEE Transactions on Pattern Analysis and Machine Intelligence},
  2017.

\bibitem{zhou2022model}
Xiao Zhou, Yong Lin, Renjie Pi, Weizhong Zhang, Renzhe Xu, Peng Cui, and Tong
  Zhang.
\newblock Model agnostic sample reweighting for out-of-distribution learning.
\newblock In {\em International Conference on Machine Learning}, pages
  27203--27221. PMLR, 2022.

\bibitem{zhou2022probabilistic}
Xiao Zhou, Renjie Pi, Weizhong Zhang, Yong Lin, Zonghao Chen, and Tong Zhang.
\newblock Probabilistic bilevel coreset selection.
\newblock In {\em International Conference on Machine Learning}, pages
  27287--27302. PMLR, 2022.

\end{thebibliography}

\clearpage

\appendix

\section{Proofs of main results}

\subsection{Proof of Theorem \ref{thm:multi_class_loss}}
\label{sect:proof_of_optimal_sampling}
\begin{proof}

\textbf{Part 1}. We first derive the optimal sampling ratio for  \textbf{coreset selection} problem.
By Lemma \ref{lemma:loss_diff}, we have
\begin{align*}
    \bbE_{\bar \cS|\cS, \pi} [\cL(\bar \beta; \cD) - \cL(\beta^*; \cD)] = \bbE_{\bar \cS|\cS, \pi} \left[ (\bar \beta - \hbeta_\MLE)^\top \bM(\beta^*; \cD) (\bar \beta - \hbeta_\MLE) + o(\|\bar \beta - \beta^*\|_2^2) \right].
\end{align*}
Therefore, we only need to find the sampling scheme which minimizes the following: 
\begin{align*}
    & \bbE_{\bar \cS|\cS, \pi} \left[ (\bar \beta - \hbeta_\MLE)^\top \bM(\beta^*; \cD) (\bar \beta - \hbeta_\MLE) \right] \\
    = & \bbE_{\bar \cS|\cS, \pi} \left[ \TRACE\left(\bM(\beta^*; \cD) (\bar \beta - \hbeta_\MLE)(\bar \beta - \hbeta_\MLE)^\top \right)\right] \\
    = & \TRACE\left(\bM(\beta^*; \cD) \bbE_{\bar \cS|\cS, \pi}\left[(\bar \beta - \hbeta_\MLE)(\bar \beta - \hbeta_\MLE)^\top\right] \right) \\
    = & \TRACE \left(\bM(\beta^*; \cD)\bM_X^{-1}(\hbeta_\MLE; \cS) \bV_{c}(\hbeta_\MLE; \cS) \bM_X^{-1}(\hbeta_\MLE; \cS) \right) \\
    = &  \frac{1}{rn^2}\sum_{(\bmx, \bmy) \in \cS} \frac{1}{\pi(\bmx, \bmy)}\TRACE \left({\psi(\hat \beta_{\MLE}; \bmx, \bmy) \otimes (\bmx \bmx^\top)}  \bM_X^{-1}(\hbeta_\MLE; \cS) \right) + o_p(1), 
\end{align*}
where the third equality is due to Lemma \ref{lemma:asympototic_variance}, and the last equality is due to Lemma \ref{lemma:closeness_between_Mxbeta} and the definition of $\bV_{c}(\hbeta_\MLE; \cS)$. Since $\sum_{(\bmx, \bmy)} \pi(\bmx, \bmy) = 1$, we have
\begin{align*}
     & \frac{1}{rn^2}\sum_{(\bmx, \bmy) \in \cS} \pi(\bmx, \bmy) \sum_{(\bmx, \bmy) \in \cS} \frac{1}{\pi(\bmx, \bmy)}\TRACE \left({\psi(\hat \beta_{\MLE}; \bmx, \bmy) \otimes (\bmx \bmx^\top)}  \bM_X^{-1}(\hbeta_\MLE; \cS) \right) \\
    \geq  & \frac{1}{rn^2} \left(\sum_{(\bmx, \bmy) \in \cS} \sqrt{\TRACE \left({\psi(\hat \beta_{\MLE}; \bmx, \bmy) \otimes (\bmx \bmx^\top)}  \bM_X^{-1}(\hbeta_\MLE; \cS) \right)}\right)^2,
\end{align*}
where the inequality is due to the Cauchy-Schwarz inequality and the equality holds when 
\begin{align*}
    \pi(\bmx, \bmy) \propto \sqrt{\TRACE \left({\psi(\hat \beta_{\MLE}; \bmx, \bmy) \otimes (\bmx \bmx^\top)}  \bM_X^{-1}(\hbeta_\MLE; \cS) \right)}.
\end{align*}

\textbf{Part 2}. We then derive the optimal sampling ratio for  \textbf{active learning} problem. We first note that
\begin{align*}
     \bbE_{\bar \cS|\cS_X, \pi}\left[\cL(\bar \beta; \cD) - \cL(\beta^*; \cD)\right] = \bbE_{ \cS|\cS_X, \pi}\left[\bbE_{\bar \cS|\cS, \pi}\left[\cL(\bar \beta; \cD) - \cL(\beta^*; \cD)\right]\right], 
\end{align*}
By lemma \ref{lemma:loss_diff}, we only need to find $\pi$ which minimizes the following equation
\begin{align*}
    & \bbE_{\cS|\cS_X, \pi}\left[\bbE_{\bar \cS|\cS, \pi} \left[ (\bar \beta - \hbeta_\MLE)^\top \bM(\beta^*; \cD) (\bar \beta - \hbeta_\MLE) \right]\right] \\
    = & \bbE_{\cS|\cS_X, \pi}\left[\frac{1}{rn^2}\sum_{(\bmx, \bmy) \in \cS} \frac{1}{\pi(\bmx)}\TRACE \left({\psi(\hat \beta_{\MLE}; \bmx, \bmy) \otimes (\bmx \bmx^\top)}  \bM_X^{-1}(\hbeta_\MLE; \cS) \right)\right] \\
    = & \frac{1}{rn^2}\sum_{(\bmx, y) \in \cS} \frac{1}{\pi(\bmx)}\TRACE \left({\bbE_{\cS|\cS_X, \pi}[\psi(\hat \beta_{\MLE}; \bmx, \bmy)] \otimes (\bmx \bmx^\top)}  \bM_X^{-1}(\hbeta_\MLE; \cS) \right) \\
    = & \frac{1}{rn^2}\sum_{(\bmx, y) \in \cS} \frac{1}{\pi(\bmx)}\TRACE \left({\phi(\hat \beta_{\MLE}; \bmx) \otimes (\bmx \bmx^\top)}  \bM_X^{-1}(\hbeta_\MLE; \cS) \right),
\end{align*}
where the last equality is due to Lemma \ref{lemma:phi_kesi_equal}. Similar to the derivation for coreset selection, the optimal sampling ratio for active learning is 
\begin{align*}
        \pi(\bmx) \propto \sqrt{\TRACE \left({\phi(\hat \beta_{\MLE}; \bmx) \otimes (\bmx \bmx^\top)}  \bM_X^{-1}(\hbeta_\MLE; \cS) \right)}.
\end{align*}
We thus finish the proof.
\end{proof}

\subsection{Proof of Theorem \ref{thm:uncertainty_estimation_multi}}
\label{app:proof_of_output_variance}
\begin{proof}
To begin with, the sample covariance matrix $\Sigma_M(\bmx)$ is computed as
\begin{align*} 
(M-1)\Sigma_M(\bmx) &= \sum_{m=1}^M (\hat \bbeta^{(m)} - \Tilde\bbeta)^\top \bmx\bmx^\top (\hat \bbeta^{(m)} - \Tilde\bbeta)\notag\\
&= \sum_{m=1}^M \big((\hat \bbeta^{(m)} - \bbeta^*) - (\Tilde\bbeta - \bbeta^*)\big)^\top \bmx\bmx^\top \big((\hat \bbeta^{(m)} - \bbeta^*) - (\Tilde\bbeta - \bbeta^*)\big)\notag\\
&= \sum_{m=1}^M (\hat \bbeta^{(m)} - \bbeta^*)^\top \bmx\bmx^\top (\hat \bbeta^{(m)} - \bbeta^*) - \frac{1}{M} \sum_{m=1}^M\sum_{l=1}^M (\hat \bbeta^{(m)} - \bbeta^*)^\top \bmx\bmx^\top (\hat\bbeta^{(l)} - \bbeta^*)\notag\\
&= \frac{M-1}{M}\sum_{m=1}^M (\hat \bbeta^{(m)} - \bbeta^*)^\top \bmx\bmx^\top (\hat \bbeta^{(m)} - \bbeta^*) - \frac{1}{M} \sum_{m\ne l} (\hat \bbeta^{(m)} - \bbeta^*)^\top \bmx\bmx^\top (\hat\bbeta^{(l)} - \bbeta^*).
\end{align*}
As $M \xrightarrow[]{} \infty$, 
\begin{align}
    \frac{1}{M}\sum_{m=1}^M (\hat \bbeta^{(m)} - \bbeta^*)^\top \bmx\bmx^\top (\hat \bbeta^{(m)} - \bbeta^*) \xrightarrow[]{} \bbE[(\hat \bbeta^{(m)} - \bbeta^*)^\top \bmx\bmx^\top (\hat \bbeta^{(m)} - \bbeta^*)] \\
    \frac{1}{M(M-1)} \sum_{m\ne l} (\hat \bbeta^{(m)} - \bbeta^*)^\top \bmx\bmx^\top (\hat\bbeta^{(l)} - \bbeta^*) \xrightarrow[]{} \bbE[(\hat \bbeta^{(m)} - \bbeta^*)^\top \bmx\bmx^\top (\hat\bbeta^{(l)} - \bbeta^*)]
\end{align}
Since $\hat \bbeta^{(m)} - \bbeta^*$ and $\hat\bbeta^{(l)} - \bbeta^*$ are independent, we have $\bbE[(\hat \bbeta^{(m)} - \bbeta^*)^\top \bmx\bmx^\top (\hat\bbeta^{(l)} - \bbeta^*)] = 0$, so as $M \xrightarrow[]{} \infty$, 
\begin{align}\label{eq:Sigma_m}
\Sigma_M(\bmx)  \xrightarrow[]{} \bbE\big[ (\hat \bbeta^{(m)} - \bbeta^*)^\top \bmx\bmx^\top (\hat \bbeta^{(m)} - \bbeta^*)\big].
\end{align}
Therefore, as $M \xrightarrow[]{} \infty$,
\begin{align*}
& s(\hbeta_\MLE; \bmx, y)^\top \Sigma_M(\bmx) s(\hbeta_\MLE; \bmx, y) \\  
\xrightarrow[]{}& s(\hbeta_{\MLE}; ; \bmx, y)^\top \bbE\big[\Sigma_M(\bmx) \big] s(\hbeta_{\MLE}; \bmx, y)\\
=& s(\hbeta_{\MLE}; \bmx, y)^\top \bbE_{\hat \bbeta^{(m)}}\big[ (\hat \bbeta^{(m)} - \bbeta^*)^\top \bmx\bmx^\top (\hat \bbeta^{(m)} - \bbeta^*) \big] s(\hbeta_{\MLE}; \bmx, y)\\
=& \TRACE\Big\{\psi(\hbeta_{\MLE}; \bmx, y)\otimes (\bmx\bmx^\top) \cdot \bbE_{\hat \bbeta^{(m)}}\big[ (\hat \bbeta^{(m)} - \bbeta^*)(\hat \bbeta^{(m)} - \bbeta^*)^\top \big]\Big\},
\end{align*}
By using the asymptotic normality of $\hat \bbeta_m$ according to Sec.8.5.2 of \cite{rice2006mathematical}:
$$
\sqrt{n}(\hat \bbeta_m-\bbeta^*) \rightarrow \cN(0, M_X^{-1}(\bbeta^*;\cD)),
$$
we obtain that as $M \xrightarrow[]{} \infty$ and $n \xrightarrow[]{} \infty$,
\begin{align}
n \TRACE(\psi(\hbeta;\bmx, y)\Sigma_M(\bmx))  &\xrightarrow[]{} \TRACE\Big\{\psi(\hbeta_{\MLE};\bmx)\otimes (\bmx\bmx^\top) M_X(\beta^*;\cD)^{-1} \Big\},
\end{align}
We finish the proof of the first part by noting $M_X(\hbeta_\MLE;\cS)$ converging to the positive definite matrix $M_X(\beta^*;\cD)^{-1}$ as shown in Lemma \ref{lemma:closeness_between_Mxbeta}.
\begin{align}
n \TRACE(\psi(\hbeta;\bmx, y)\Sigma_M(\bmx))  &\xrightarrow[]{} \TRACE\Big\{\psi(\hbeta_{\MLE};\bmx)\otimes (\bmx\bmx^\top) M_X(\hat \beta_\MLE;\cS)^{-1} \Big\}.
\end{align}
\end{proof}

\section{Supporting Lemmas}
\begin{lemma} 
\label{lemma:loss_diff}
For any subset $\bar \cS$ and subsampling estimator $\bar \beta$ yielded by some subsampling probability $\pi$, we have
    \begin{align*}
        \bbE_{\bar \cS|\cS, \pi} [\cL(\bar \beta; \cD) - \cL(\beta^*; \cD)] = \bbE_{\bar \cS|\cS, \pi} \left[ (\bar \beta - \hbeta_\MLE)^\top \bM(\beta^*; \cD) (\bar \beta - \hbeta_\MLE) + o(\|\bar \beta - \beta^*\|_2^2) \right]
    \end{align*}
\end{lemma}
\begin{proof}
        We obtain
     \begin{align*}
        \cL(\bar \beta; \cD) - \cL(\beta^*; \cD) = &  (\bar\beta - \beta^*)^\top\frac{\partial \cL(\beta; \cD)}{\partial \beta} \Big |_{\beta=\beta^*} + (\bar\beta - \beta^*)^\top \left(\frac{\partial^2 \cL(\beta; \cD)}{\partial \beta^2} \Big |_{\beta=\beta^*}\right) (\bar\beta - \beta^*) + o(\|\bar \beta - \beta^*\|_2^2)
    \end{align*}
    Since $\frac{\partial \cL(\beta; \cD)}{\partial \beta} \Big |_{\beta=\beta^*} = 0,$
    and 
    $
        \frac{\partial^2 \cL(\beta; \cD)}{\partial \beta^2} \Big |_{\beta=\beta^*} = \bM(\beta^*; \cD)
    $
    , we have 
    \begin{align*}
        \cL(\bar \beta; \cD) - \cL(\beta^*; \cD) 
        = (\bar \beta - \beta^*)^\top \bM(\beta^*; \cD) (\bar \beta - \beta^*) + o(\|\bar \beta - \beta^*\|_2^2).
    \end{align*}
    Further,  
    \begin{align*}
        &  \bbE_{\bar \cS|\cS, \pi} \left[(\bar \beta - \beta^*)^\top \bM(\beta^*; \cD) (\bar \beta - \beta^*)\right] \\
        = & \bbE_{\bar \cS|\cS, \pi} \left[(\bar \beta - \hbeta_\MLE + \hbeta_\MLE -  \beta^*)^\top \bM(\beta^*; \cD) (\bar \beta - \hbeta_\MLE + \hbeta_\MLE -  \beta^*)\right] \\
        = &  \bbE_{\bar \cS|\cS, \pi} \left[ (\bar \beta - \hbeta_\MLE)^\top \bM(\beta^*; \cD) (\bar \beta - \hbeta_\MLE)\right] + 2 \bbE_{\bar \cS|\cS, \pi} \left[( \beta^* - \hbeta_\MLE)^\top \bM(\beta^*; \cD) (\bar \beta - \hbeta_\MLE) \right] \\
            & + \bbE_{\bar \cS|\cS, \pi} \left[ (\bar \beta^* - \hbeta_\MLE)^\top \bM(\beta^*; \cD) (\beta^* - \hbeta_\MLE)\right] 
    \end{align*}
    We have 
\begin{align*}
    \bbE_{\bar \cS|\cS, \pi} \left[( \beta^* - \hbeta_\MLE)^\top \bM(\beta^*; \cD) (\bar \beta - \hbeta_\MLE) \right] = ( \beta^* - \hbeta_\MLE)^\top \bM(\beta^*; \cD) \bbE_{\bar \cS|\cS, \pi}[(\bar \beta - \hbeta_\MLE)] = 0 
\end{align*}
So we have
    \begin{align*}
        \argmin_{\pi} \bbE_{\bar \cS|\cS, \pi} [\cL(\bar \beta; \cD) - \cL(\beta^*; \cD)] = \argmin_{\pi}\bbE_{\bar \cS|\cS, \pi} \left[ (\bar \beta - \hbeta_\MLE)^\top \bM(\beta^*; \cD) (\bar \beta - \hbeta_\MLE) + o(\|\bar \beta - \beta^*\|_2^2) \right]
    \end{align*}
by noting that $\bbE_{\bar \cS|\cS, \pi} \left[ (\bar \beta^* - \hbeta_\MLE)^\top \bM(\beta^*; \cD) (\beta^* - \hbeta_\MLE)\right]$ is independent of $\pi$.
\end{proof}
\begin{lemma}
\label{lemma:phi_kesi_equal}
    We have
    \begin{align*}
        \bbE_{\by|\bx}[\psi_{k_1k_2}(\beta^*; \bx, \by)] = \phi_{k_1k_2}(\beta^*; \bx).
    \end{align*}
\end{lemma}
\begin{proof}
When $k_1=k_2=k$, 
we have
\begin{align*}
    & \bbE_{\by|\bx} [\psi(\beta^*; \bx, \by)] \\
    = & \bbE_{\by|\bx} [(\delta_{k}(y) - p_{k}(\beta^*; \bx))^2] \\
    = & p_k(\beta^*; \bx) (1-p_{k}(\beta^*; \bx))^2 + (1-p_k (\beta^*; \bx)) p_{k}(\beta^*; \bx)^2 \\
    = & p_k(\beta^*; \bx) - p_k(\beta^*; \bx)^2 = \phi_k(\beta^*; \bx). 
\end{align*}
If $k_1 \neq k_2$, we have 
\begin{align*}
    & \bbE_{\by|\bx} [\psi_{k_1k_2}(\beta_\MLE; \bx, \by)] \\
    = & \bbE_{\by|\bx} [(\delta_{k_1}(\by) - p_{k_1}(\beta^*; \bx)) (\delta_{k_2}(\by) - p_{k_2}(\beta^*; \bx))] \\
    = & \bbE_{\by|\bx} [\delta_{k_1}(\by)\delta_{k_2}(\by)  - \delta_{k_2}(\by)p_{k_1}(\beta^*; \bx) -  \delta_{k_1}(\by)p_{k_2}(\beta^*; \bx) +  p_{k_1}(\beta^*; \bx)p_{k_2}(\beta^*; \bx)) ] \\
    = &  p_{k_2}(\beta^*; \bx)p_{k_1}(\beta^*; \bx) -  p_{k_1}(\beta^*; \bx)p_{k_2}(\beta^*; \bx) +  p_{k_1}(\beta^*; \bx)p_{k_2}(\beta^*; \bx) \\
    = & - p_{k_1}(\beta^*; \bx)p_{k_2}(\beta^*; \bx) = \phi_{k_1 k_2}(\beta^*; \bx)
\end{align*}
\end{proof}
\begin{lemma}[Variance of $\bar \beta$, Theorem 1 of \cite{yao2019optimal}]
\label{lemma:asympototic_variance}
    If Assumptions \ref{ass:positive_M} and \ref{ass:bounded_momentum} hold, as $n\rightarrow\infty$ and $r\rightarrow\infty$, condition on $\cS_{X,2}$ in probability,
$$
\bV^{-1/2}(\bar \beta - \hat \beta_{\MLE}) \rightarrow N(0, \bI),
$$
where
\begin{gather*}
\bV = \bM_X^{-1}(\hbeta_\MLE; \cS) \bV_{c}(\hbeta_\MLE; \cS) \bM_X^{-1}(\hbeta_\MLE; \cS),\\
\bM_X(\hbeta_\MLE; \cS) = \frac{1}{n} \sum_{(\bmx, \bmy) \in \cS} \phi(\hat \beta_{\MLE};\bmx) \otimes (\bmx \bmx^\top), \\
 \bV_{c}(\hbeta_\MLE; \cS) = \frac{1}{rn^2}\sum_{(\bmx, \bmy) \in \cS} \frac{\psi(\hat \beta_{\MLE}; \bmx) \otimes (\bmx \bmx^\top)}{\pi(\bmx, \bmy)}.
\end{gather*}
\end{lemma}

\begin{lemma}\label{lm:matrix_conv}
Consider a finite sequence $\{\bX_i\}_{i=1}^n$ of independent, $D\times D$ random matrices with the same expectation $\bbE \bX_i := \bar M_X$. Let $M_X=\frac{1}{n}\sum_{i=1}^n\bX_i$. If we assume that $\lambda_{\min}(\bar M_X) \wedge \lambda_{\min}(M_X)\ge \kappa_{\min}$ and $\lambda_{\max}(\bar M_X)\vee \lambda_{\max}(M_X)\le \kappa_{\max}$ for some constant $\kappa_{\min},\kappa_{\max}>0$, we have with probability at least $1-\delta$,
$$
\|M_X^{-1}-\bar M_X^{-1}\|_{\op} \le \kappa_{\min}^{-2} \kappa_{\max}\sqrt{\frac{8\log(D/\delta)}{n}}.
$$
\end{lemma}
\begin{proof}
We deduce that
\begin{align*}
\|M_X^{-1}-\bar M_X^{-1}\|_{\op} &= \left\|M_X^{-1}(\bar M_X-M_X)\bar M_X^{-1}\right\|_{\op}\\
&\le \left\|M_X^{-1}\right\|_{\op} \cdot \left\|\bar M_X-M_X\right\|_{\op} \cdot \left\|\bar M_X^{-1}\right\|_{\op}\\
&\le \kappa_{\min}^{-2}\cdot \left\|\bar M_X-M_X\right\|_{\op},
\end{align*}
where the last inequality is due to $\lambda_{\min}(M_X) \ge \kappa_{\min}$ as $n\rightarrow\infty$. 
Then, it remains to show that $\|\bar M_X-M_X\|_{\op}$ goes to zero.

Since $M_X^2 \preceq \kappa_{\max}^2$, and
$$
\bbE\left[M_X-\bar M_X\right] = \bbE\left[\frac{1}{n}\sum_{i=1}^n\bX_i - \bbE\bX\right] = 0,
$$
we apply matrix Hoeffding inequality from Theorem 1.3 in \cite{tropp2012user} to obtain that for all $t\ge0$,
\begin{align*}
\bbP\left(\left\|M_X-\bar M_X\right\|_{\op}\ge t\right) \le D\cdot e^{-nt^2/8\kappa_{\max}^2}.
\end{align*}
By taking $t=\kappa_{\max}\sqrt{8\log(D/\delta)/n}$ for any $\delta>0$, we get with probability at least $1-\delta$
\begin{align*}
\left\|M_X-\bar M_X\right\|_{\op} \le \kappa_{\max}\sqrt{\frac{8\log(D/\delta)}{n}},
\end{align*}
which completes the proof.
\end{proof}

\begin{lemma} \label{lemma:closeness_between_Mxbeta}
Assuming that for any $\bmx\in\cD$, $\|\bmx\|_2\le L$, and $\bM_X(\hat\beta_\MLE; \cS)$ is positive definite: $\bM_X(\hat\beta_\MLE; \cS)\succeq \nu I$ for some $\nu>0$. When $n\ge 32L^4\log(d/\delta)/\nu^2$, for any $\bmx\in\cD$, with probability at least $1-2\delta$,
\begin{align*}
\bmx^\top \left(  \bM_X(\hat\beta_\MLE; \cS)-\bM_X(\beta^*; \cD)\right) \bmx = O(n^{-1/2}),
\end{align*}
where $\bM_X(\hat\beta_\MLE; \cS)=n^{-1}\sum_{\bmx\in\cS_x}w(\beta_{\MLE},\bmx)\bmx\bmx^\top$, and $\bM_X(\beta^*; \cD)=\bbE_{(\bx,\by) \sim \cD}[w(\beta_*,\bx)\bx\bx^\top]$.
\end{lemma}
\begin{proof}
We can get $\lambda_{\max}(\bM_X(\beta^*; \cS))\le n^{-1}\sum_{\bmx\in\cS_x}\lambda_{\max}(\bmx\bmx^\top)\le \|\bmx\|_2^2 \le L^2$, and $\lambda_{\max}(\bM_X(\beta^*; \cD))\le L^2$. Thus, from Lemma \ref{lm:matrix_conv}, we know that $\bM_X(\beta^*; \cS)$ converges to $\bM_{X}(\beta^*; \cD)$ in probability, i.e., with probability at least $1-\delta$,
\begin{align*}
\left\|\bM_X(\beta^*; \cS)-\bM_{X}(\beta^*; \cD)\right\|_{\op} \le L^2\sqrt{\frac{8\log(d/\delta)}{n}} \le \frac{\nu}{2},
\end{align*}
where the second inequality is obtained since $n\ge 32L^4\log(d/\delta)/\nu^2$. Therefore, it follows that $\bM_{X}(\beta^*; \cD)\succeq\nu I/2$. Then, conditionling on $\bM_{X}(\beta^*; \cD)\succeq\nu I/2$, we can invoke Lemma \ref{lm:matrix_conv} with $\kappa_{\min}=\nu/2$ and $\kappa_{\max}=L^2$ to obtain for any $\bmx\in\cD_X$, with probability at least $1-\delta$,
\begin{align}\label{eq:matrix_appro_1}
\bmx^\top \left( M_{X}^{-1}(\beta^*; \cS)-\bM_X^{-1}(\beta^*; \cD)\right) \bmx = O(n^{-1/2}).
\end{align}
Additionally, because $\beta_{\MLE}-\beta^*=O_P(n^{-1/2})$ and $\|\bmx\|_2\le L$ for any $\bmx\in\cD$, we have
\begin{align*}
\left\|\bM_X(\beta_{\MLE}; \cS) - \bM_X(\beta^*; \cS)\right\|_{\op} \le \frac{1}{n} \sum_{\bmx\in\cS_x}\big|w(\beta_{\MLE},\bmx)-w(\beta^*,\bmx)\big| L^2 \le \frac{L^2}{\sqrt{n}},
\end{align*}
which indicates that $\bM_X(\beta^*; \cS) \ge \nu I/2$ due to $n\ge 4L^4/\nu^2$. Therefore, we get
\begin{align}
\label{eq:matrix_appro_2}
\left\|\bM_X^{-1}(\beta_{\MLE}; \cS) - \bM_X^{-1}(\beta^*; \cS)\right\|_{\op} &\le \left\|\bM_X^{-1}(\beta_{\MLE}; \cS)\right\|_{\op} \cdot \left\|\bM_X(\beta^*; \cS)-\bM_X(\beta_{\MLE}; \cS)\right\|_{\op}\cdot \left\|\bM_X^{-1}(\beta^*; \cS)\right\|_{\op}\\
&\le \frac{2}{\nu^2}\cdot\frac{L^2}{n} = O(n^{-1/2}).
\end{align}
Combining the result in \eqref{eq:matrix_appro_1} and \eqref{eq:matrix_appro_2}, we obtain that with probability at least $1-2\delta$,
\begin{align*}
&\bmx^\top \left( \bM_X^{-1}(\beta_{\MLE}; \cS)-\bM_X^{-1}(\beta^*; \cD)\right) \bmx\\
=& \bmx^\top \left( \bM_X^{-1}(\beta_{\MLE}; \cS) - \bM_X^{-1}(\beta^*; \cS) + \bM_X^{-1}(\beta^*; \cS) -\bM_X^{-1}(\beta^*; \cD)\right) \bmx\\
=& O_P(n^{-1/2}),
\end{align*}
which concludes the proof.
\end{proof}

\section{Algorithms}
\subsection{Uncertainty Sampling Algorithm for DNNs}
\label{sect:uncertainty_estimation_dnns}
\begin{algorithm}[H]
\caption{COPS  for coreset selection on DNNs \label{algo:dnn_those_coreset}}
\KwInput{Training data $\cS$, one probe datasets $\cS'$, sub-sampling size $r$. }
\KwOut{The selected subset $\bar\cS$ and the model $\bar \theta$.}

For each $(\bmx, \bmy) \in \cS$, obtain $u(\bmx, \bmy)$ by Algorithm~\ref{algo:uncertainty_estimation_DNN} with the probe set $\cS'$;
 
Randomly draw $\bar\cS$ containing $r$ samples from $\cS$ by $u(\bmx, \bmy) / \sum_{(\bmx', \bmy') \in \cS} u(\bmx', \bmy')$.

Solve $\bar \beta$ on the weighted subset $\bar\cS(\pi)$ as follows.
\begin{align}
    \small
    \label{eqn:weighted_solver_dnn}
    \bar{\theta}  = \argmin_\theta \left(-\frac{1}{r} \sum_{(\bar\bmx, \bar \bmy) \in \bar \cS } \frac{1}{\pi(\bar \bmx, \bar \bmy)} \left( \sum_{k=1}^K\delta_{k}(\bar \bmy)\bar f_k(\theta; \bar\bmx) - \log\{1+\sum_{l=1}^K \exp(f_l(\theta; \bar\bmx))\} \right) \right).
\end{align}

\end{algorithm}

\begin{algorithm}[H]
\caption{COPS  for active learning on DNNs \label{algo:dnn_those_active}}
\KwInput{Training data $\cS_X$, one probe datasets $\cS'$, sub-sampling size $r$. }
\KwOut{The selected subset $\bar\cS$ with inquired label and the model $\bar \theta$.}

For each $\bmx \in \cS_X$, obtain $u(\bmx)$ by Algorithm~\ref{algo:uncertainty_estimation_DNN} with the probe set $\cS'$;
 
Randomly draw $\bar\cS_X$ containing $r$ samples from $\pi(\bmx) = \cS$ by $u(\bmx) / \sum_{\bmx' \in \cS} u(\bmx')$.

Obtain the labeled data set $\bar\cS$ by labeling each sample in $\bar\cS_X$.

Solve $\bar \beta$ on the weighted subset $\bar\cS(\pi)$ as follows
\begin{align}
    \small
    \label{eqn:weighted_solver_active_dnn}
    \bar{\theta}  = \argmin_\theta \left(-\frac{1}{r} \sum_{\bar\bmx \in \bar \cS_X } \frac{1}{\pi(\bar \bmx)} \left( \sum_{k=1}^K\delta_{k}(\bar \bmy)\bar f_k(\theta; \bar\bmx) - \log\{1+\sum_{l=1}^K \exp(f_l(\theta; \bar\bmx))\} \right) \right).
\end{align}
\end{algorithm}


\begin{algorithm}[H]
\caption{Uncertainty estimation for DNNs.  \label{algo:uncertainty_estimation_DNN}}
\KwInput{Probe datasets $\cS'$, the sampling dataset $\cS$ for coreset selection or $\cS_X$ for active learning.}
\KwOut{The estimated uncertainty for each sample in $\cS$ or $\cS_X$.}
For $m = 1, ..., M$, randomly initialize $f_{\theta^{(m)}}$ with different seeds and then minimize the loss of $f_{\theta^{(m)}}$ on $\cS'$ by SGD independently.

For each $\bmx$, obtain the output logits of each model $ \{f_{\theta^{(m)}}(\bmx)\}_{m=1}^M$ and the covariance of the logits:
\begin{align}
    \label{eqn:covariance_of_logits_DNN}
    \Sigma_M(\bmx) = \frac{1}{M - 1}\sum_{m=1}^M\left( f_{\theta^{(m)}}(\bmx)  - 
        \frac{1}{M} \sum_{l=1}^M f_{\theta^{(l)}}(\bmx) \right) \left( f_{\theta^{(m)}}(\bmx)  - 
        \frac{1}{M} \sum_{l=1}^M f_{\theta^{(l)}}(\bmx) \right) ^\top.
\end{align}

Get the predicted probability of sample $\bmx$ by $\frac{1}{M} \sum_{m=1}^M p(f_{\theta^{(m)}}; \bmx)$. Estimate the uncertainty for each sample 
\begin{itemize}
    \item Case (1) coreset selection. Obtain $\psi(\Tilde{\beta}; \bmx, \bmy)$ according to Eqn~\eqref{eqn:psi_defi} and obtain the uncertainty the estimation
    \begin{align*}
        u(\bmx, \bmy) = \TRACE\left(\psi(\Tilde\beta;\bmx, \bmy)\Sigma_M(\bmx)\right);
    \end{align*}
    \item Case (2) active learning. Obtain $\phi(\Tilde{\beta}; \bmx)$ according to Eqn~\eqref{eqn:phi_defi} and obtain the uncertainty the estimation as 
    \begin{align*}
 u(\bmx) = \TRACE\left(\phi(\Tilde\beta;\bmx)\Sigma_M(\bmx)\right).
    \end{align*} 
\end{itemize}

\end{algorithm}

\subsection{COPS-clip Algorithm for DNNs}
\label{app:algorithms_with_clip}
\begin{algorithm}[H]
\caption{COPS with uncertainty clipping  for coreset selection on DNNs \label{algo:dnn_those_coreset_clipping}}
\KwInput{Training data $\cS$, one probe datasets $\cS'$, sub-sampling size $r$, hyper-parameter $\alpha$. }
\KwOut{The selected subset $\bar\cS$ and the model $\bar \theta$.}

For each $(\bmx, \bmy) \in \cS$, obtain $u(\bmx, \bmy)$ by Algorithm~\ref{algo:uncertainty_estimation_DNN} with the probe set $\cS'$;
 
Randomly draw $\bar\cS$ containing $r$ samples from $\cS$ by 
\begin{align*} \pi^\alpha(\bmx, \bmy) = \min\left\{\alpha, u(\bmx, \bmy)\right\} / \sum_{(\bmx', \bmy') \in \cS} \min\left\{\alpha, u(\bmx', \bmy')\right\}.\end{align*}

Calculate the reweighting of each selected sample $(\bar\bmx, \bar\bmy)$ as $\frac{1}{\pi(\bar\bmx, \bar\bmy)}$, where 
\begin{align*}
    \pi(\bmx, \bmy) = u(\bmx, \bmy) / \sum_{(\bmx', \bmy') \in \cS}  u(\bmx', \bmy').
\end{align*}

Solve $\bar \theta$ on the weighted subset as:$$\bar \theta = \argmin_\theta \left(-\frac{1}{r} \sum_{(\bar\bmx, \bar \bmy) \in \bar \cS } \frac{1}{\pi(\bar\bmx, \bar \bmy)} \left( \sum_{k=1}^K\delta_{k}(\bar \bmy) f_{\theta, k}(\bar \bmx)- \log\{1+\sum_{l=1}^K \exp(f_{\theta, l}(\bar \bmx))\} \right) \right)
$$

\end{algorithm}

\begin{algorithm}[H]
\caption{COPS with uncertainty clipping  for active learning on DNNs \label{algo:dnn_those_active_clipping}}
\KwInput{Training data $\cS_X$, one probe datasets $\cS'$, sub-sampling size $r$, hyper-parameter $\alpha$. }
\KwOut{The selected subset $\bar\cS$ with inquired label and the model $\bar \theta$.}

For each $\bmx \in \cS_X$, obtain $u(\bmx)$ by Algorithm~\ref{algo:uncertainty_estimation_DNN} with the probe set $\cS'$;
 
Randomly draw $\bar\cS_X$ containing $r$ samples from $\cS$ by
$$\pi^\alpha(\bmx) = \min\left\{\alpha, u(\bmx)\right\} / \sum_{\bmx' \in \cS_X} \min\left\{\alpha, u(\bmx')\right\}.$$

Calculate the reweighting of each selected sample $\bar\bmx$ as $\frac{1}{\pi(\bar\bmx)}$, where 
\begin{align*}
    \pi(\bmx) = u(\bmx) / \sum_{(\bmx') \in \cS_X}  u(\bmx').
\end{align*}

Obtain the labeled data set $\bar\cS$ by labeling each sample in $\bar\cS_X$

Solve $\bar \theta$ on the weighted subset as:$$
    \bar \theta = \argmin_\theta \left(-\frac{1}{r} \sum_{(\bar\bmx, \bar \bmy) \in \bar \cS } \frac{1}{\pi(\bar \bmx)} \left( \sum_{k=1}^K\delta_{k}(\bar \bmy) f_{\theta, k}(\bar \bmx)- \log\{1+\sum_{l=1}^K \exp(f_{\theta, l}(\bar \bmx))\} \right) \right).
$$
\end{algorithm}

\section{Experimental Details}
\subsection{Details of the experiment in Section~\ref{sect:main_experiments} }
\label{app:exp_details}
We evaluate the performance of the model on the original test set. We use AdamW Optimizer~\cite{loshchilov2018decoupled} with cosine lr decay for 150 epochs, the batch size is 256. We put a limit on the maximum weight when solving Eqn~\eqref{eqn:weighted_solver} to avoid large variance. Specifically, let $u_i$ denote the uncertainty of $i$th sample. In Eqn~\eqref{eqn:weighted_solver}, we use $\frac{1}{u_i}$ to reweight the selected data. To avoid large variance, we use  $\frac{1}{\max\{\beta, u_i\}}$ as the reweighting to replace $\frac{1}{u_i}$. We simply set $\beta=0.1$ for all experiments following \cite{citovsky2023leveraging}. So the coreset selection algorithm with full details are shown in Algorithm~\ref{algo:dnn_those_coreset_clipping_withbeta} and \ref{algo:dnn_those_active_clipping_withbeta}. Comparing Algorithm~\ref{algo:dnn_those_coreset_clipping_withbeta}-\ref{algo:dnn_those_active_clipping_withbeta} with  Algorithm~\ref{algo:dnn_those_coreset_clipping}-\ref{algo:dnn_those_active_clipping}, we can see the only difference is that we use $\pi_\beta$ instead of $\pi$ to re-weight the selected data, which is consistent with \cite{citovsky2023leveraging}. We use  Algorithm~\ref{algo:dnn_those_coreset_clipping_withbeta} and \ref{algo:dnn_those_active_clipping_withbeta} in the main experiment part by default. 

\begin{algorithm}[H]
\caption{COPS with full details for coreset selection on DNNs \label{algo:dnn_those_coreset_clipping_withbeta}}
\KwInput{Training data $\cS$, one probe datasets $\cS'$, sub-sampling size $r$, hyper-parameter $\alpha$. }
\KwOut{The selected subset $\bar\cS$ and the model $\bar \theta$.}

For each $(\bmx, \bmy) \in \cS$, obtain $u(\bmx, \bmy)$ by Algorithm~\ref{algo:uncertainty_estimation_DNN} with the probe set $\cS'$;
 
Randomly draw $\bar\cS$ containing $r$ samples from $\cS$ by 
\begin{align*} \pi^\alpha(\bmx, \bmy) = \min\left\{\alpha, u(\bmx, \bmy)\right\} / \sum_{(\bmx', \bmy') \in \cS} \min\left\{\alpha, u(\bmx', \bmy')\right\}.\end{align*}

Calculate the reweighting of each selected sample $(\bar\bmx, \bar\bmy)$ as $\frac{1}{\pi_\beta(\bar\bmx, \bar\bmy)}$, where 
\begin{align*}
    \pi_\beta(\bmx, \bmy) = \max\{\beta, u(\bmx, \bmy)\} / \sum_{(\bmx', \bmy') \in \cS}  \max\{\beta,u(\bmx', \bmy')\}.
\end{align*}

Solve $\bar \theta$ on the weighted subset as:$$\bar \theta = \argmin_\theta \left(-\frac{1}{r} \sum_{(\bar\bmx, \bar \bmy) \in \bar \cS } \frac{1}{\pi_\beta(\bar\bmx, \bar \bmy)} \left( \sum_{k=1}^K\delta_{k}(\bar \bmy) f_{\theta, k}(\bar \bmx)- \log\{1+\sum_{l=1}^K \exp(f_{\theta, l}(\bar \bmx))\} \right) \right)
$$

\end{algorithm}

\begin{algorithm}[H]
\caption{COPS with full details for active learning on DNNs \label{algo:dnn_those_active_clipping_withbeta}}
\KwInput{Training data $\cS_X$, one probe datasets $\cS'$, sub-sampling size $r$, hyper-parameter $\alpha$. }
\KwOut{The selected subset $\bar\cS$ with inquired label and the model $\bar \theta$.}

For each $\bmx \in \cS_X$, obtain $u(\bmx)$ by Algorithm~\ref{algo:uncertainty_estimation_DNN} with the probe set $\cS'$;
 
Randomly draw $\bar\cS_X$ containing $r$ samples from $\cS$ by
$$\pi^\alpha(\bmx) = \min\left\{\alpha, u(\bmx)\right\} / \sum_{\bmx' \in \cS_X} \min\left\{\alpha, u(\bmx')\right\}.$$

Calculate the reweighting of each selected sample $\bar\bmx$ as $\frac{1}{\pi_\beta(\bar\bmx)}$, where 
\begin{align*}
    \pi_\beta(\bmx) = \max\{ \beta, u(\bmx)\} / \sum_{(\bmx') \in \cS_X}  \max\{ \beta, u(\bmx')\}.
\end{align*}

Obtain the labeled data set $\bar\cS$ by labeling each sample in $\bar\cS_X$

Solve $\bar \theta$ on the weighted subset as:$$
    \bar \theta = \argmin_\theta \left(-\frac{1}{r} \sum_{(\bar\bmx, \bar \bmy) \in \bar \cS } \frac{1}{\pi_\beta(\bar \bmx)} \left( \sum_{k=1}^K\delta_{k}(\bar \bmy) f_{\theta, k}(\bar \bmx)- \log\{1+\sum_{l=1}^K \exp(f_{\theta, l}(\bar \bmx))\} \right) \right).
$$
\end{algorithm}
\subsection{Hyper-parameters of the experiments in Section~\ref{sect:main_experiments}}

\begin{table}[H]
    \centering
    \small
    \resizebox{\linewidth}{!}{
    \begin{tabular}{c|cccccc}
    \toprule
        Dataset & CIFARBinary & CIFAR10/CIFAR10-N & CIFAR100 & SVHN & Places365 & IMDB\\
        \midrule
        Class Number & 2&10&100&10&10&2\\
        \hline
        Size of the probe set  & 2,000 & 10,000& 20,000& 10,000&10,000&5,000\\
         Start learning rate 1 & 0.1 & 0.1 & 0.1 & 0.1 &0.1 &0.1\\
        Learning rate schedule 1 &schedule 1 & schedule 1 & schedule 1 & schedule 1 & 
 schedule 1& no schedule\\
         Optimizer 1 & SGD& SGD& SGD& SGD& SGD& AdamW\\
        Epoch 1&100 & 100 & 100 & 100 & 100& 20 \\
        
        \hline
        Size of the sampling set &8,000 & 40,000& 30,000 &63,257&40,000&20,000 \\
         Start learning rate 2 & 0.1 & 0.1 & 0.1 & 0.1 &0.1 &0.1\\
        Learning rate schedule 2 &schedule 2 & schedule 2 & schedule 2 & schedule 2 & 
 schedule 1& no schedule\\
         Optimizer 2 & AdamW& AdamW& AdamW& AdamW& SGD& AdamW\\
        Epoch 2&150&150&150&150&100&20 \\
         \hline
        Size of the testing set &2,000 &10,000& 10,000 & 26,032& 1,000 & 25,000\\

        \bottomrule
    \end{tabular}
    }
    \caption{This table illustrates the training details. Here we set weight decay as 5e-4 for all the experiments. Here no schedule means using the start learning rate without modification during training.  Schedule 1 stands for the decaying of the learning rate by 0.1 every 30 epochs. Schedule 2 means using the cosine learning schedule with $T_{max}=50$ and $eta_{min}=0$}
    \label{tab:my_label}
\end{table}
\subsection{Structure of models for IMDB}
\label{sec:archi}
We adopt GRU for IMDB, whose structure is shown as follows:
\begin{table}[H]
    \centering
    \begin{tabular}{c|c}
        layer&    GRU Model\\
            \hline
        1&   Embedding(2000, 200)\\
       2&  Dropout(p=0.3)\\
       3&  GRU(hidden\_size= 24, num\_layers=2,  dropout=0.3, bidirectional=True)\\
       4& Maxpool() \& Avgpool()\\
       5& Concat(last,max, avg)\\
       6&  Linear(in\_features=72, out\_features=1, bias=True)
    \end{tabular}
    \caption{Model structure for GRU}
    \label{tab:my_label}
\end{table}

\section{More experimental results}
\subsection{Label noise\label{sec:label_noise}}
\noindent We found that COPS-vanilla selects large number of samples with large uncertainty, which is later shown to exacerbate label noise when sub-sampling a dataset with natural label noise, CIFAR10-N. 
\begin{table}[H]
    \centering
    \begin{tabular}{c|c|c|c}
\hline 
Dataset & Sampling Method & WithY & WithoutY\tabularnewline

\hline 
\multirow{3}{*}{CIFAR10-N} & Uniform & 0.0872 &0.0912\tabularnewline
\cline{2-4} \cline{3-4} \cline{4-4} 
 & COPS-vanilla & 0.1314 & 0.0934\tabularnewline
\cline{2-4} \cline{3-4} \cline{4-4} 
 & COPS-clip & 0.0941 & 0.0906\tabularnewline
\hline 
\end{tabular}
    \caption{Noise ratio comparison for different sampling methods. Here we sample 1000 instances from each class and use the uncertainty with the label known. The noise ratio of uniform sampling for  the coreset selection (WithY) and active learning (WithoutY) is slightly different. This is because we sample within each class in coreset selection. For example, in the CIFAR10-N-1000-WithY, we uniformly select 100 samples in each class. However, in the CIFAR10-N-1000-WithoutY, we uniformly select 1000 samples in the whole dataset. \label{tab:noise_ratio}}
    
\end{table}

\subsection{Comparison of threshold in the sampling and reweighting stages\label{sec:thre_exp}}

Let $u$ represent $u(\bmx, \bmy)$ for coreset selection and $u(\bmx)$ for active learning. The COPS method consists of two stages:

\begin{itemize}
    \item Stage 1: Data sampling based on $u$. To prevent COPS from oversampling low-density data, we propose limiting the maximum value of $u$ by $\min\{\alpha, u\}$ in this stage. (Section~\ref{sect:thresholding_uncertainty})
    \item Stage 2: Weighted learning, where each selected sample is assigned a weight of $1/u$ to obtain an unbiased estimator. To reduce variance, \cite{ionides2008truncated, swaminathan2015counterfactual, citovsky2023leveraging} propose adding a threshold of $1/\max\{\beta, u\}$ in this stage.
\end{itemize}

We refer to the thresholding method in the sampling stage as ``$\alpha$-clip" and the thresholding method in the reweighting stage as ``$\beta$-clip". We investigate the impact of ``$\alpha$-clip" and ``$\beta$-clip" on the final performance. We compare four methods based on whether thresholding is applied in the first and second stages:

\begin{itemize}
    \item ``Vanilla sampling + vanilla reweighting": No thresholding is used in either stage.
    \item ``$\alpha$-clip sampling + vanilla reweighting": Thresholding of $\min\{\alpha, u\}$ is applied in the sampling stage, while reweighting remains unchanged.
    \item ``Vanilla sampling + $\beta$-clip reweighting": Sampling stage remains unchanged, but a threshold of $1/\max\{\beta, u\}$ is used in the reweighting stage.
    \item ``$\alpha$-clip sampling + $\beta$-clip reweighting": Both sampling and reweighting stages utilize thresholding methods.
\end{itemize}

The comprehensive results are displayed in Figure~\ref{fig:2_stage_clip_cmpr}. The results clearly demonstrate that both the utilization of $\alpha$-clip in sampling and $\beta$-clip in reweighting lead to performance improvement. Importantly, it is observed that the performance gain of each method cannot be solely attributed to the other. The optimal performance is achieved by effectively combining the benefits of both techniques.

\begin{figure}[H]
\centering
\begin{minipage}[t]{0.48\textwidth}
\centering
\includegraphics[width=6cm]{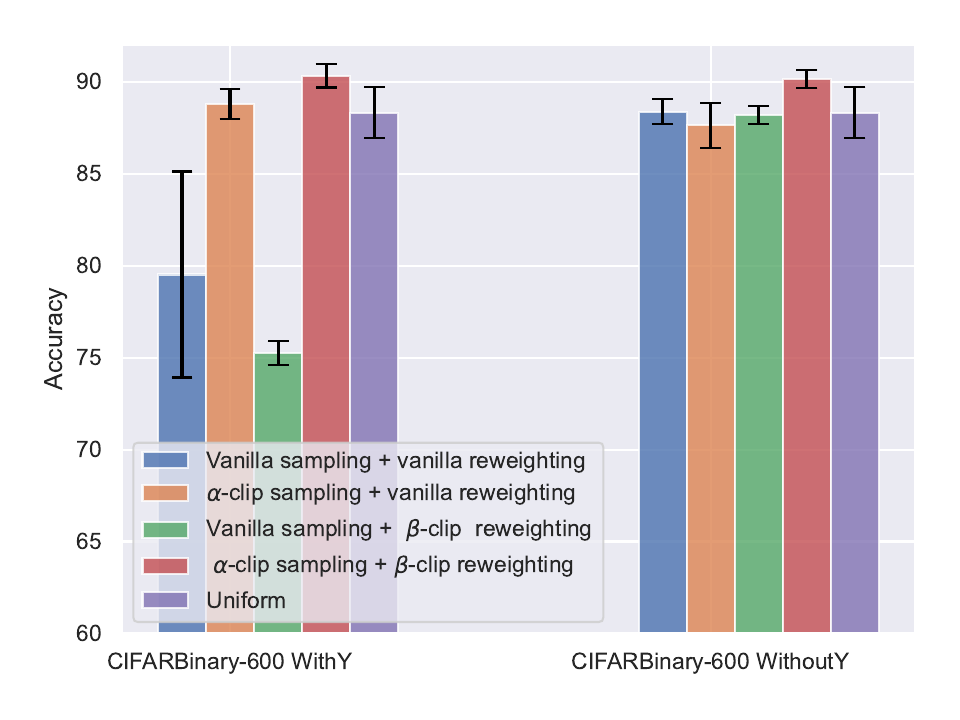}
\caption*{(a) CIFARBinary-600.\label{300_pilot_detailed_cifar_binary}}
\end{minipage}
\begin{minipage}[t]{0.48\textwidth}
\centering
\includegraphics[width=6cm]{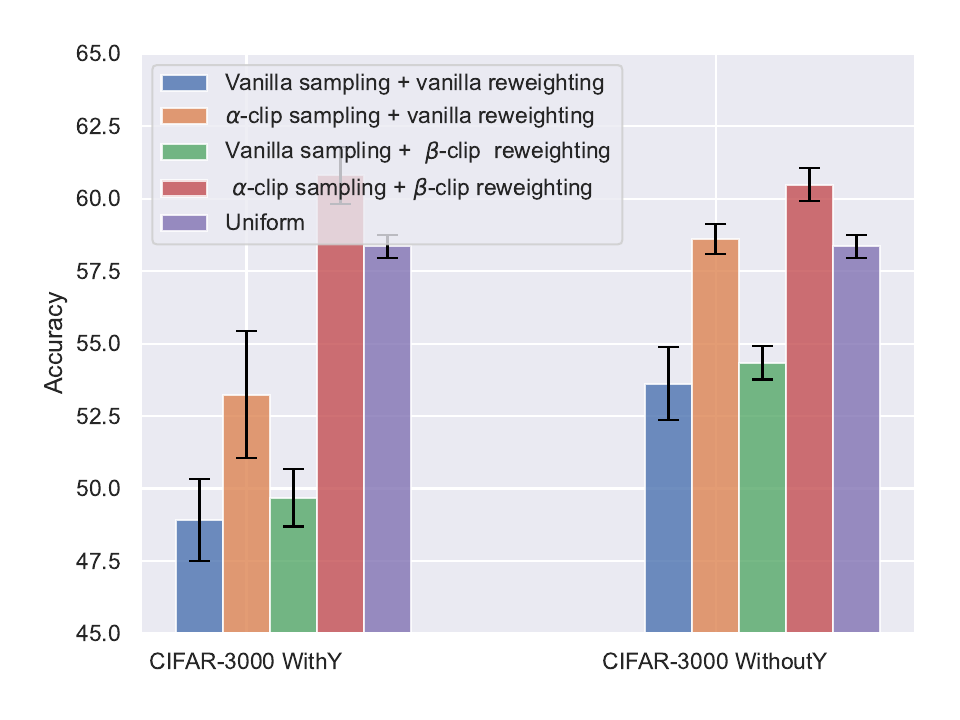}
\caption*{(b)CIFAR10-3000.\label{300_pilot_detailed_cifar} }
\end{minipage}
\caption{Results of comparing the $\alpha$-clip in the first stage (sampling stage) and $\beta$-clip in the second stage (reweighting stage) }
\label{fig:2_stage_clip_cmpr}
\end{figure}



\subsection{Comparison with full data}

\paragraph{CIFAR10Binary.} The figure for CIFAR10Binary is shown in Figure.\ref{cifar10binary_full_compa}.

\begin{figure}[H]
\centering
\begin{minipage}[t]{0.48\textwidth}
\centering
\includegraphics[width=6cm]{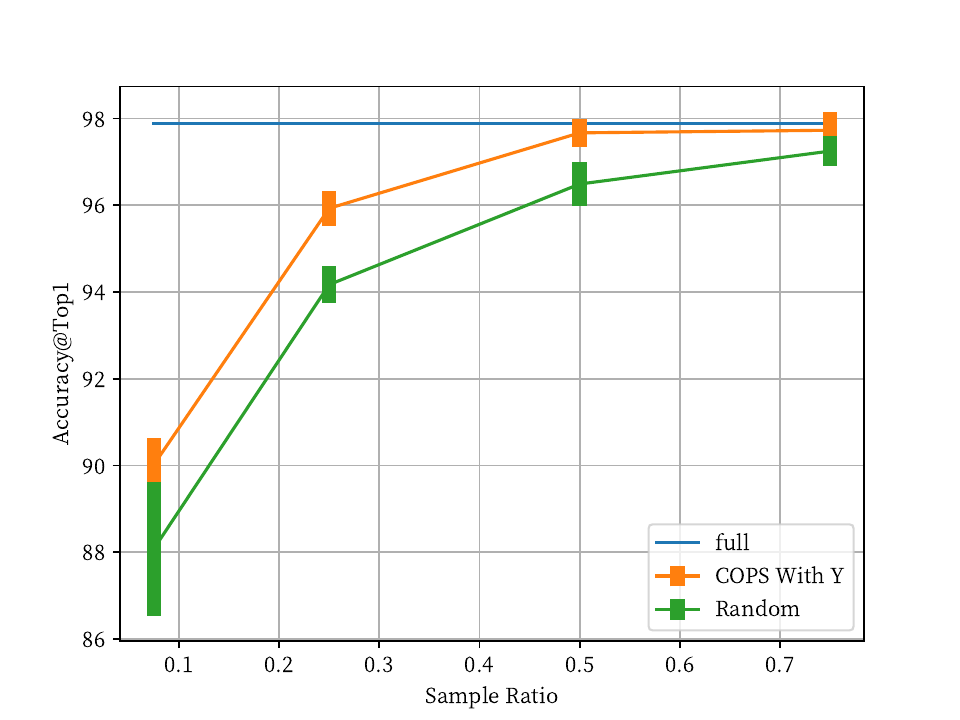}
\caption*{(a)Comparison on CIFAR10Binary with Y.}
\end{minipage}
\begin{minipage}[t]{0.48\textwidth}
\centering
\includegraphics[width=6cm]{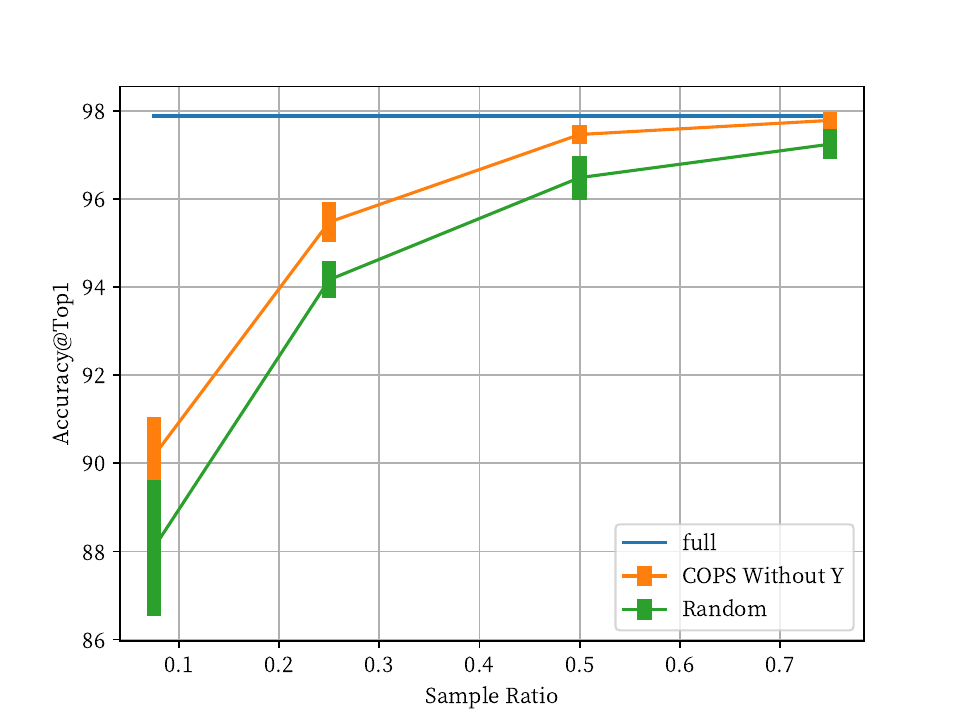}
\caption*{(b)Comparison on CIFAR10Binary without Y.}
\end{minipage}
\caption{Comparison with full dataset.~\label{cifar10binary_full_compa}}
\end{figure}

\paragraph{CIFAR10.} The figure for CIFAR10 is shown in Figure.\ref{cifar10_full_compa}.

\begin{figure}[H]
\centering
\begin{minipage}[t]{0.48\textwidth}
\centering
\includegraphics[width=6cm]{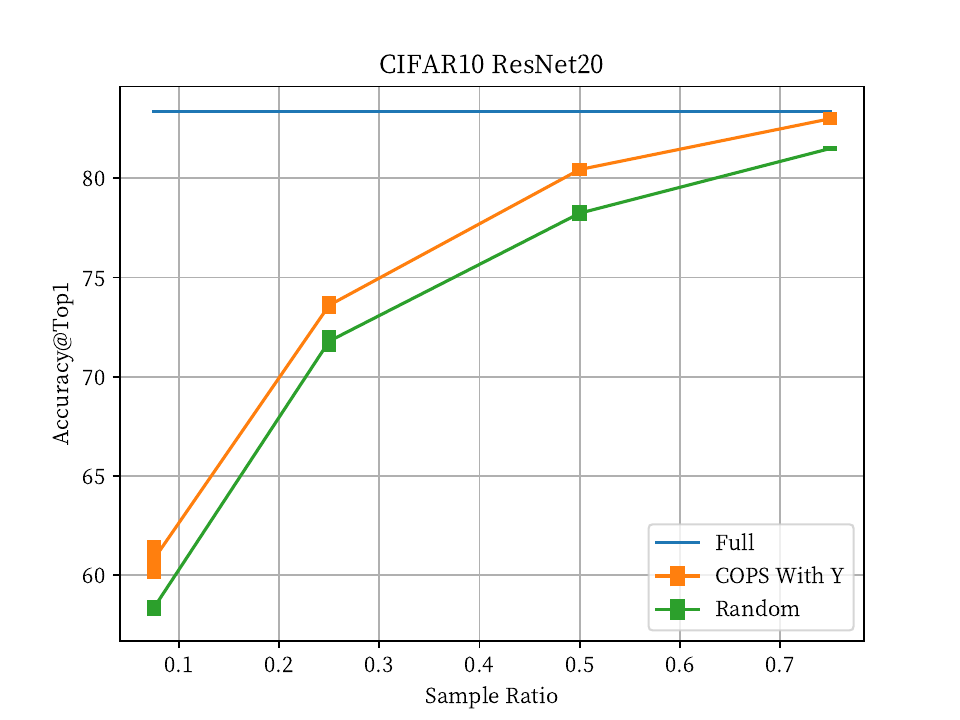}
\caption*{(a)Comparison on CIFAR10 with Y.}
\end{minipage}
\begin{minipage}[t]{0.48\textwidth}
\centering
\includegraphics[width=6cm]{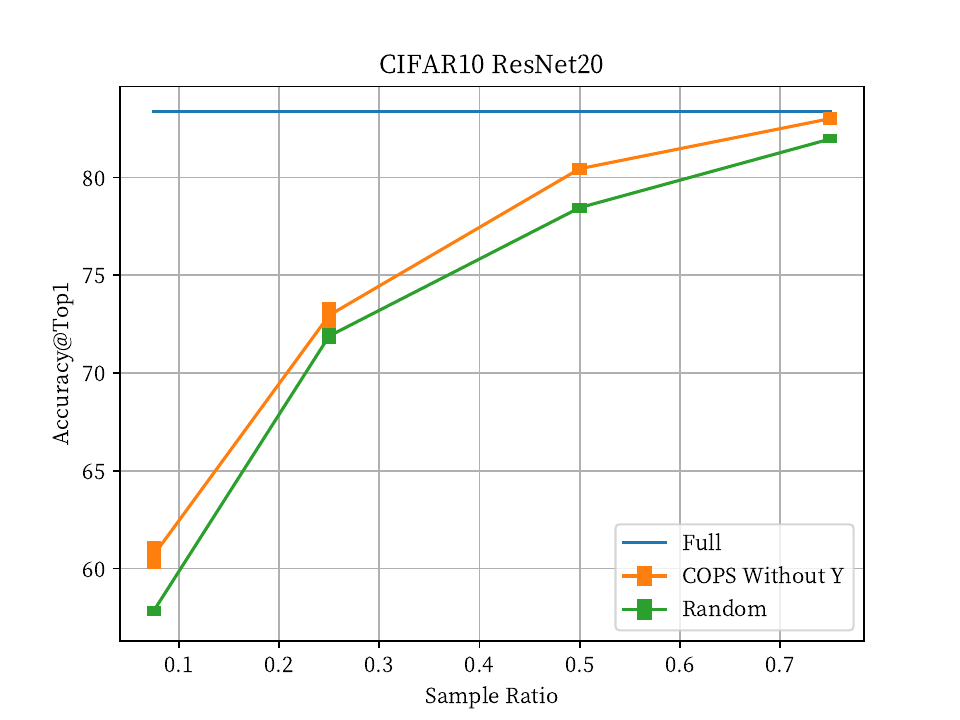}
\caption*{(b)Comparison on CIFAR10 without Y.}
\end{minipage}
\caption{Comparison with full dataset.~\label{cifar10_full_compa}}
\end{figure}

\end{document}